\documentclass[10pt,journal,compsoc]{IEEEtran}

\usepackage{booktabs}
\usepackage{amsthm}
\theoremstyle{plain}
\usepackage[compress]{cite}
\usepackage{graphics}
\usepackage{epstopdf}
\usepackage{epsfig}
\usepackage{lineno,hyperref}
\usepackage{array}
\usepackage{amsmath,amssymb}
\usepackage{mathrsfs}
\usepackage{subfigure} 
\usepackage{multirow}
\usepackage{xcolor}
\usepackage{bm}
\usepackage{eqparbox}
\usepackage{algorithm}  
\usepackage{algorithmic}
\usepackage{graphicx}
\usepackage{amsfonts}
\usepackage{pifont}
\usepackage{enumerate}
\usepackage{indentfirst}
\usepackage{url}
\usepackage{ragged2e}

\newtheorem{myTheo}{Theorem}
\newtheorem{myCorollary}{Corollary}
\newtheorem{myLemma}{Lemma}
\newtheorem{myAssumption}{Assumption}
\newtheorem{myDefinition}{Definition}
\newtheorem{myProposition}{Proposition}

\usepackage{bm}
\usepackage{bbm}
\usepackage{pifont}
\newcommand{\xmark}{\ding{55}}
\usepackage{algorithm}  
\usepackage{algorithmic}

\allowdisplaybreaks[4]

\begin{document}

\title{Matrix Completion via Non-Convex Relaxation and Adaptive Correlation Learning}

\author{Xuelong Li, \IEEEmembership{~Fellow,~IEEE}, Hongyuan Zhang, and Rui Zhang, \IEEEmembership{~Member,~IEEE}



\thanks{
    The authors are with the School of Artificial Intelligence, OPtics and ElectroNics (iOPEN), Northwestern Polytechnical University, Xi'an 710072, P.R. China. 
    They are also with the Key Laboratory of Intelligent Interaction and Applications (Northwestern Polytechnical University), Ministry of Industry and Information Technology, Xi'an 710072, P. R. China.
}

\thanks{
    This work is supported by The National Natural Science Foundation of China
(No. 61871470).
}


\thanks{Corresponding author: R. Zhang and X. Li}

\thanks{E-mail: li@nwpu.edu.cn; hyzhang98@gmail.com; ruizhang8633@gmail.com}

}

\markboth{IEEE TRANSACTIONS ON pattern analysis and machine intelligence}{Li \MakeLowercase{\textit{et al.}}: Matrix Completion via Non-Convex Relaxation and Adaptive Correlation Learning}

\IEEEtitleabstractindextext{

\begin{abstract}
    \justifying
    The existing matrix completion methods focus on optimizing the relaxation of rank function 
    such as nuclear norm, Schatten-$p$ norm, \textit{etc}. 
    They usually need many iterations to converge. 
    Moreover, only the low-rank property of matrices is utilized in most 
    existing models and several methods that incorporate other knowledge 
    are quite time-consuming in practice. 
    To address these issues, we propose a novel non-convex surrogate that 
    can be optimized by closed-form solutions, such that it empirically 
    converges within dozens of iterations. 
    Besides, the optimization is parameter-free and the convergence is proved. 
    Compared with the relaxation of rank, 
    the surrogate is motivated by optimizing an upper-bound of rank. 
    We theoretically validate that it is equivalent to the existing 
    matrix completion models. 
    Besides the low-rank assumption, we intend to exploit the column-wise 
    correlation for matrix completion, 
    and thus an adaptive correlation learning, which is scaling-invariant, 
    is developed.
    More importantly, after incorporating the correlation learning, 
    the model can be still solved by closed-form solutions such that 
    it still converges fast. 
    Experiments show the effectiveness of the non-convex surrogate and 
    adaptive correlation learning. 
\end{abstract}

\begin{IEEEkeywords}
    Artificial Intelligence, Pattern Recognition, Matrix Completion, Non-Convex Surrogate, Adaptive Correlation Learning, 
    Parameter-Free Optimization.
\end{IEEEkeywords}

}

\maketitle

\section{Introduction}
Matrix is a fundamental element in machine learning and the low-rank property 
of matrix has been applied in many practical applications \cite{image_restore,MC-literature-1,application_radar}. 
Low-rank matrix completion (LRMC) \cite{nuclear,nuclear-2}, 
aiming to recover a low-rank matrix according to the 
observed entries,
plays an important role in many fields, such as 
image recovery \cite{application_inpainting,ImageRecovery}, 
recommendation systems \cite{application_recommendation}, 
robust principal component analysis \cite{RPCA,TensorRPCA,RPCA-Application}, 
multi-task learning \cite{MTL-1,MTL-2,FMC-MTL}, \textit{etc}.
The main motivation behind low-rank is from the observation that a part of 
principal components of a matrix usually contain most of the information, 
especially in optical imagery (shown in Figure \ref{figure_low_rank}). 
In other words, the distribution of singular values of an image matrix is often 
heavy-tailed.

The original LRMC model \cite{nuclear} intends to optimize the nuclear norm 
of matrix, the convex envelope of rank. To improve the performance, 
plenty of works focus on optimizing the nuclear norm and its variants such 
as truncated version \cite{TNNR}, weighted version \cite{WNNM}, 
\textit{etc}. 
As the nuclear norm is the $\ell_1$-norm of singular values, 
the nuclear norm relaxed problem can be generalized to the Schatten-$p$ norm. 
With $0 < p < 1$, the Schatten-$p$ norm approximates rank 
better than the nuclear norm. 
It should be emphasized that Schatten-$p$ is not convex when $0 < p < 1$.
In particular, solutions of LRMC with Schatten-1/2 and Schatten-2/3 are 
derived in \cite{Sp-norm-solution-1,Sp-norm-solution-2,Sp-norm-solution-3}.
Additionally, several different surrogates are 
also developed to obtain a better approximation of 
rank \cite{log,extra_surrogate}. 
Besides, diverse factorization models are derived from these surrogates. 
For instance, the factored nuclear norm \cite{F-nuclear-application-1,F-nulcear-application-2} 
transforms it into two Frobenius norm terms while RegL1 \cite{RegL1} 
tries to improve the robustness of the noisy model. 
Factored group-sparse regularization \cite{FGSR} proves that the Schatten-p 
norm is equivalent to the sum of two group-sparse norms. 
Bilinear model \cite{bilinear} shows that the Schatten-$p$ norm can be 
converted into nuclear norms of two factored matrices. 
To improve the results of matrix completion, 
the models proposed in \cite{mc_graph,mc_graph_1,mc_graph_2} incorporate the similarity as 
the prior information. 
However, the similarity is given as the prior information such that 
the model fails to work on general cases. 
Besides, several works \cite{LRFD,S3LR} focus on how to integrate 
various kinds of information. The main barrier of these hybrid models 
is inefficient optimization. 
They usually consume significant amounts of time to train, 
which results in unavailability in practice.

To optimize the proposed models, 
several optimization techniques are applied such as the 
semidefinite programming (\textit{SDP}) \cite{nuclear}, 
augmented Lagrange multiplier method (\textit{ALM}) \cite{RegL1}, 
alternative direction method of multipliers 
(\textit{ADMM}) \cite{ADMM}, 
re-weighted method \cite{re-weighted-1,re-weighted-2}, \textit{etc}. 
SDP is time-consuming especially when $m$ and $n$ are not tiny values 
(\textit{e.g.}, $m = n = 100$) \cite{survey}. 
ALM \cite{ALM} and ADMM \cite{ADMM} are the most popular methods 
in matrix completion but it needs lots of iterations to converge. 
To accelerate the optimization, auxiliary variables are frequently 
introduced \cite{RegL1,FGSR} and the linearized ADMM \cite{LADMM} 
is widely applied since the direct subproblem of ADMM may have no 
closed-form solution. 
Non-factored models usually depend on singular value decomposition 
(\textit{SVD}) which causes computational complexity $O(m n^2)$ 
per iteration. 
Factored models usually need $O(m n d)$ time per iteration where 
$d$ is the column number of factored matrices. 
For the noisy extension, the proximal gradient method is widely used 
to solve the non-smooth term, \textit{i.e.}, 
the nuclear norm (denoted by $\|\cdot\|_*$). 
However, $d$ is hard to set and experiments show that smaller 
$d$ may lead to slower convergence in some cases. 
In other words, ADMM requires more iterations to converge, 
which implies expensive costs.

In sum, the existing LRMC models rely on parametric algorithms 
(\textit{e.g.,} gradient-based methods, ADMM-based methods, \textit{etc.}),
which require lots of iterations to converge. 
Additionally, most of them only focus on the low-rank property.
Different from the existing models, 
we propose a  model with a novel \textbf{N}on-\textbf{C}onvex surrogate and \textbf{A}daptive cor\textbf{R}elation \textbf{L}earning (\textit{NCARL}) for LRMC problem, 
to achieve faster convergence and exploit the hidden information of the matrix.
Besides the low-rank assumption, NCARL incorporates an adaptive correlation learning mechanism to exploit correlation and
mines the potential information column-wisely.
The main contributions are listed as follows:
\begin{itemize}
    \item We aim to optimize an upper-bound of ${\rm rank}(\cdot)$ via the full-rank factorization, 
            which provides a novel non-convex surrogate. 
            Compared with other factored methods, our model does not need the initial rank $d$. 
            Surprisingly, it can be solved by closed-form solutions without linearization and auxiliary variables,
            such that its optimization is totally parameter-free.
    \item Besides the low-rank assumption, the potential information of 
            columns is exploited by our model via learning 
            column-wise correlation adaptively. 
            Owing to the smooth surrogate, 
            the model still has closed-form solutions after incorporating 
            the adaptive correlation learning, 
            which implies the two parts are compatible.
    \item The proposed algorithm usually converges within 20 iterations such 
            that it is competitive in the terms of efficiency compared with factored models, 
            even though our model does not need the initial rank.
            Although factored models seem to need less time per iteration, 
            they require a large number of iterations to converge.
            Experiments support the computational efficiency of our model.
\end{itemize}

\subsection{Notations} 
In this paper, $\bm m^i$ and $\bm m_j$ denote the $i$-th row and $j$-th column of $M$, respectively. 
$M^\dag$ is the Moore-Penrose pseudo-inverse. 
$\mathcal{R}(M)$ represents the space spanned by columns of $M$.
${\rm diag}(\bm m)$ represents the diagonal matrix with diagonal entries $\bm m$. 
$\mathbb{S}_{+}^n$ and $\mathbb{S}_{++}^n$ denote the set of positive semi-definite and positive definite $n \times n$ matrices, respectively. 
$\odot$ represents the Hadamard product. Without additional statements, given a matrix $M \in \mathbb R^{m \times n}$, we assume $m \geq n$. Given a square matrix $Q \in \mathbb R^{n \times n}$ and non-zero binary vectors $\bm p, \bm q = \{0, 1\}^n$ where $\|\bm p\|_0 = k_1$ and $\|\bm q\|_0 = k_2$, $[Q]_{\bm p, \bm q} \in \mathbb R^{k_1 \times k_2}$ represents the sub-matrix where the $i$-th row and $j$-th column are deleted from $Q$ if $p_i = 0$ and $q_j = 0$. 
Specially, for a vector $\bm v$, $[\bm v]_{\bm p}$ represents the sub-vector that removes the $i$-th entry from $\bm v$ if $\bm v_i = 0$. $\textbf{1}$ denotes the vector whose entries are all 1 and $\bar {\bm p} = \textbf{1} - \bm p$. 
$\ell_{2,0}$-norm and $\ell_{2,1}$-norm of $M$ are respectively defined as 
$\|M\|_{2,0} = \sum_{i=1}^m \mathbbm{1}[\|\bm m^i\|_2 \neq 0]$
and $\|M\|_{2,1} = \sum_{i=1}^m \|\bm m^i\|_2$ where $\bm m^i$ is the 
$i$-th row of $M$. 
All proofs are summarized in appendix.

\begin{figure}[t]
    \centering
    \subfigure[One $320 \times 240$ image]{
        \includegraphics[width=0.47\linewidth]{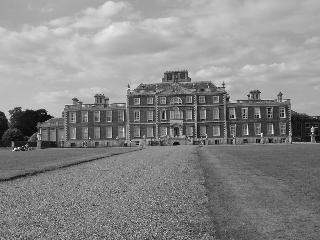}
    }
    \subfigure[Distribution of singular values]{
        \includegraphics[width=0.47\linewidth]{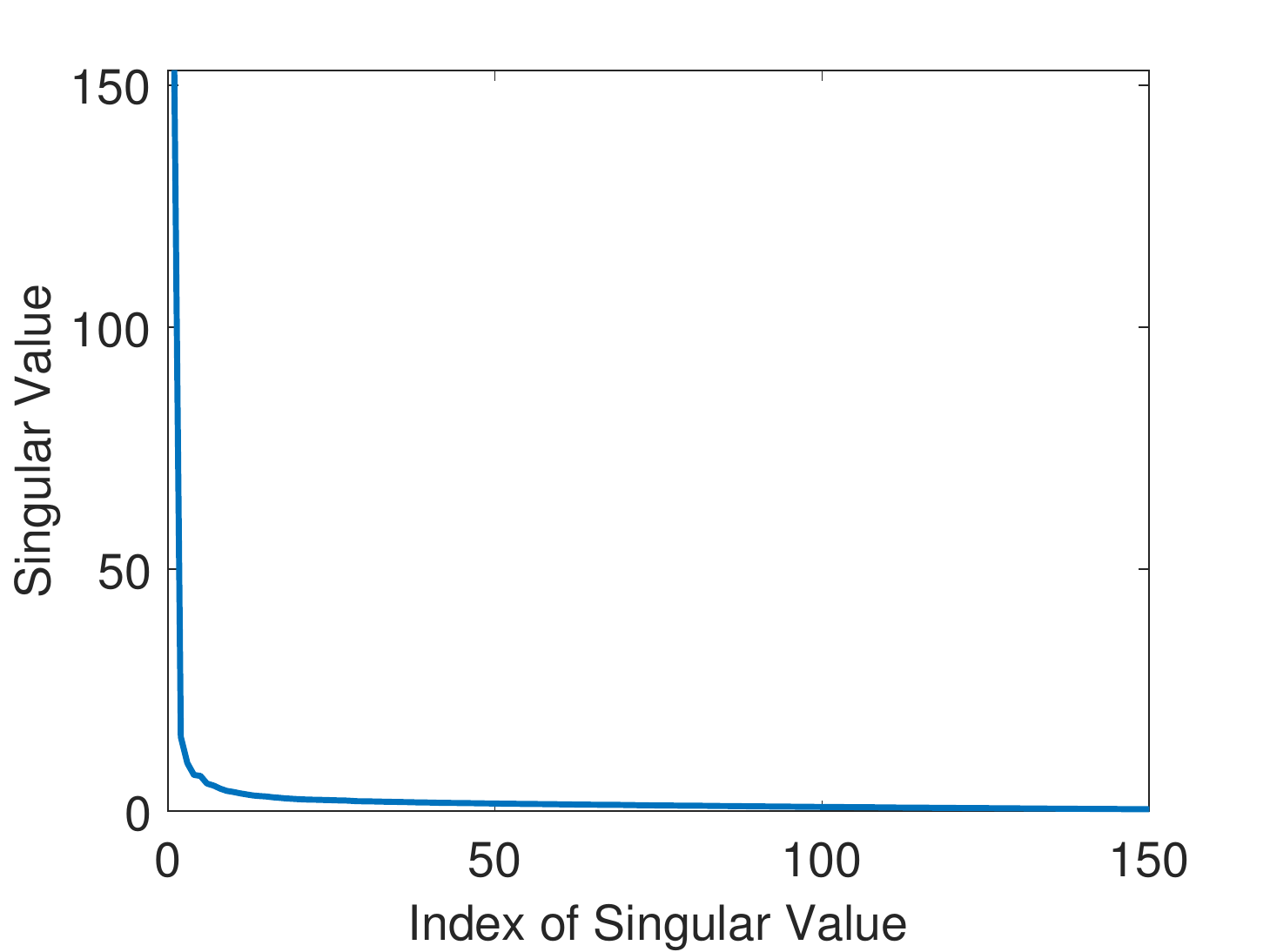}
    }
    \caption{An illustration of singular values of a natural image.}
    \label{figure_low_rank}
\end{figure}

\section{Related Work}
First, we provide the formal definition of LRMC. 
Suppose that we have observed some entries denoted by a matrix 
$M \in \mathbb{R}^{m \times n}$, where indexes are represented by $\Omega$. 
Particularly, $M_{ij} = 0$ if $(i, j) \notin \Omega$.
The known LRMC attempts to recover a low-rank matrix $X$ 
from the observations. 
LRMC can be formulated as 
\begin{equation}
    \min \limits_{X} {\rm rank}(X), ~~~~ s.t.~ X_{ij} = M_{ij} ~~ \forall (i, j) \in \Omega .
\end{equation}
Assume that $m \geq n$ holds, then we could rewrite the constraint as a concise matrix form via introducing a filter matrix P as 
\begin{equation}
    P_{ij} = 
    \left \{
    \begin{array}{l l}
        1, & (i, j) \in \Omega ; \\
        0, & (i, j) \notin \Omega .
    \end{array}
    \right .
\end{equation}
Accordingly, the LRMC problem can be rewritten as  
\begin{equation} \label{obj_rank}
    \min \limits_{X} {\rm rank}(X), ~~~~ s.t. ~~ X \odot P = M ,
\end{equation}
where $\odot$ represents the Hadamard product.
Instead of solving the NP-hard problem caused by ${\rm rank}(\cdot)$, 
the existing models \cite{nuclear,Schatten-p,FGSR} aim to optimize a relaxed function $\varphi(\cdot)$ 
of ${\rm rank}(\cdot)$ such that the objective is converted into
\begin{equation}
    \min \limits_{X} \varphi (X), ~~~~ s.t. ~ X \odot P = M .
\end{equation}
For instance, the classical LRMC model \cite{nuclear} uses the nuclear norm, $\|\cdot\|_*$, 
as $\varphi(\cdot)$. 
Max norm \cite{RecoveryBounds-2} is also investigated for LRMC.
As $\|X\|_* = \sum_{i=1} \sigma_i$ where $\sigma_i$ 
represents the singular value, an important variant is to employ 
the Schatten-$p$ norm \cite{Schatten-p,FGSR,Sp-norm-solution-1,Sp-norm-solution-2,Sp-norm-solution-3}, 
$\|X\|_{S_p} = (\sum_{i=1} \sigma_i^p)^{\frac{1}{p}}$.
Based on these surrogates, some models design more complicated surrogates, 
such as truncated nuclear norm \cite{TNNR}, weighted nuclear norm \cite{WNNM},
\textit{etc}.
Besides, several models \cite{S3LR,LRFD,mc_graph} also integrate other 
kinds of information. S$^3$LR introduces the popular subspace exploitation 
into the mechanism, while the prior graph information is utilized in the 
model proposed by \cite{mc_graph}. 
To accelerate the optimization and avoid searching the rank of matrix in 
all possible values, factored models \cite{F-nuclear-application-1,F-nulcear-application-2,FGSR} 
have been widely investigated.
The core idea of factored models is to assume that recovered matrix 
can be factored as two small matrices. For example, the factored nuclear norm
model \cite{F-nuclear-application-1} is defined as 
\begin{equation}
    \min \limits_{X = AB, X \odot P = M} \frac{1}{2}(\|X\|_*) \Leftrightarrow
    \min \limits_{AB \odot P = M} \frac{1}{2}(\|A\|_F^2 + \|B\|_F^2) .
\end{equation}
FGSR \cite{FGSR} aims to factorize Schatten-1/2 and Schatten-2/3 as 
two convex surrogates instead.

A well-known extension of LRMC is the noisy version.
If the contaminated case regarding polluted observations is considered, 
\textit{i.e.}, $M_{ij} = (X_{ij})_* + \varepsilon_{ij}$ where $(X_{ij})_*$ and $\varepsilon_{ij}$ denote the true value and noise respectively, the recovery task indicates the simultaneous minimization of residuals and rank,
\begin{equation}
    \min \limits_{X} \|X \odot P - M\|_F^2 + \gamma \cdot {\rm \varphi}(X) ,
\end{equation}
where $\gamma$ is the hyper-parameter to leverage residuals and rank. 
Some works \cite{RegL1,TNNR} focus on improving the noisy model as well.
Specifically, RegL1 \cite{RegL1} utilizes $\ell_1$-norm to replace the 
Frobenius norm and ensure $\{\varepsilon_{ij}\}_{i,j}$ sparse.
To retrieve a simple discussion,  we only focus on the noiseless case at 
first and the model can be easily extended into the noisy case.

\section{Problem Reformulation} \label{section_mc}

Unlike most LRMC models, we do not relax ${\rm rank}(\cdot)$ as the nuclear norm. 
Instead,  
for an arbitrary matrix $X \in \mathbb R^{m \times n}$, we can apply full-rank factorization and have
\begin{equation}
    X = W^T U^T ,
\end{equation}
where $U \in \mathbb R^{n \times n}$ is an orthogonal matrix. Note that the full-rank factorization is not unique. On the one hand, ${\rm rank}(X) = {\rm rank}(W) \leq \|W\|_{2, 0}$ since ${\rm rank}(W) \leq \|W\|_{2,0}$ holds for any $W$. In other words, $\ell_{2,0}$-norm is an upper-bound of ${\rm rank}(\cdot)$. Therefore, we can optimize the following upper-bound as the objective function
\begin{equation}
    \label{obj_20}
    \begin{split}
     \min \limits_{W, U} & ~ \|W\|_{2,0}, ~~~~ s.t. ~~ (W^T U^T) \odot P = M, U^T U = I .
    \end{split}
\end{equation}
On the other hand, $\{\bm u_i\}_{i=1}^n$ denotes an orthonormal basis 
while $\bm w_j$ can be regarded as the coordinate of $(\bm x^j)^T$ under 
$\{\bm u_i\}_{i=1}^n$. 
The low-rank property of $X$ indicates that only a few basis vectors are activated. 
In sum, $\|W\|_{2,0}$ is a rational replacement of ${\rm rank}(X)$.
The following theorem rigorously shows the connection between the following 
problem and the original one. 
\begin{myTheo} \label{theo_equivalence_20_rank}
    Problem (\ref{obj_20}) is equivalent to problem (\ref{obj_rank}). 
    In other words, $X_* = W_*^T U_*^T$ where $X_*$, $W_*$, and $U_*$ are 
    the optimal solutions of two problems, respectively.
\end{myTheo}
Accordingly, we can focus on how to optimize problem (\ref{obj_20}).
Likewise, since the optimization of $\ell_{2,0}$-norm is NP-hard and $\ell_{2,1}$-norm is convex envelope of it, problem (\ref{obj_20}) can be relaxed into 
\begin{equation} \label{obj_21}
    \begin{split}
        & \min \limits_{W, U} \| W \|_{2,1}, ~~~~ s.t. ~~ (W^T U^T) \odot P = M, U^T U = I; \\
        \Rightarrow & \min \limits_{W, U} \| W \|_{2,1}^2, ~~~~ s.t. ~~ (W^T U^T) \odot P = M, U^T U = I .
    \end{split}
\end{equation}
The following theorem demonstrates that the above objective function can be converted into a smooth function which has continuous first-order derivative.  
\begin{myTheo} \label{theo_equivalence}
    Define $\Psi$ and $\Psi'$ as follows 
    \begin{equation}
        \left \{ 
        \begin{array}{l}
            \Psi = \{(X, D) | X \odot P = M, {\rm tr}(D^\dag) = 1, D \in \mathbb{S}_{+}^n\}, \\
            \Psi' = \{(W, U) | (W^T U^T) \odot P = M, U^T U = I\}.
        \end{array}
        \right .
    \end{equation}
    Then problem (\ref{obj_21}) is equivalent to 
    \begin{equation} \label{obj_no_graph}
        \begin{split}
            \min \limits_{X, D} ~ &  {\rm tr}(X D X^T), 
            ~~~~ s.t. ~ (X, D) \in \Psi .
        \end{split}
    \end{equation}
    Meanwhile, the relation of optimal solutions of the two problems 
    can be established as  
    \begin{equation}
        X_* = W_*^{T} U^T_*, D_* = U_* \Lambda U_*^T, \Lambda = {\rm diag}(\frac{\|\bm w^i_*\|_2}{\|W_*\|_{2,1}})^\dag, 
    \end{equation}
    where $U_* \Lambda U_*^T$ is the eigenvalue decomposition of $D_*$,
    \begin{equation}
        \left \{
        \begin{split}
            & (X_*, D_*) = 
            \arg \min \limits_{(X, D) \in \Psi} {\rm tr}(X D X^T), \\
            & (W_*, U_*) = 
            \arg \min \limits_{(W, U) \in \Psi'} \|W\|_{2,1}^2 .
        \end{split}
        \right .
    \end{equation}

\end{myTheo}
\begin{myProposition} \label{proposition}
    Problem (\ref{obj_no_graph}) is non-convex. 
    In particular, the subproblem regarding $X$ is convex.
\end{myProposition}
In spite of the non-convexity, ${\rm tr}(X D X^T)$ is smooth compared with 
$\|X\|_*$. In the following subsection, an efficient gradient-free 
algorithm, which can converge into the global optimum, is developed.
Besides, in the next section, we will find that this surrogate is more 
compatible with additional mechanisms.

\subsection{Optimization of Problem (\ref{obj_no_graph})}
Since the problem is non-convex and the subproblem regarding $X$ is convex, 
we optimize problem (\ref{obj_no_graph}) by an alternative method. 
Inspired by \cite{MTL-2}, 
Theorem \ref{theo_solution_D} provides a closed-form solution for the subproblem 
regarding $D$.
\begin{myTheo} \label{theo_solution_D}
    If $X$ is fixed as constant, the optimum of problem (\ref{obj_no_graph}) is $\|X\|_*^2$, \textit{i.e.},
    \begin{equation}
    \begin{split}
        & \|X\|_*^2 = \min \limits_{D} {\rm tr}(X D X^T), ~~ s.t. ~ {\rm tr}(D^\dag) = 1, D \in \mathbb{S}_{+}^n ,
    \end{split}
    \end{equation}
    where the optimal $D$ is given as
    \begin{equation} \label{solution_D}
        D = (\frac{(X^T X)^{\frac{1}{2}}}{{\rm tr}((X^T X)^{\frac{1}{2}})})^\dag .
    \end{equation}
\end{myTheo}

\begin{algorithm}[t]
    \caption{Algorithm to solve problem (\ref{obj_no_graph}).}
    \label{alg_surrogate}
    \begin{algorithmic}
        \REQUIRE Mask matrix $P$, observed entries $M$, perturbation coefficient 
        $\delta = 10^{-6}$, and maximum iterations $t_{m} = 50$.
        \STATE $X \leftarrow M$.
        \REPEAT
            \STATE Update $D$ by Eq. (\ref{solution_D}).
            \STATE $\hat D \leftarrow D + \delta I$.
            \STATE $F_i \leftarrow P_i \hat D^{-1} P_i$.
            \STATE Update $X$ according to Eq. (\ref{solution_X_noiseless}): $\bm x^i \leftarrow \bm m^i (F_i)^{\bm p^i +} \hat D^{-1}$.
        \UNTIL{convergence or exceeding maximum iteration $t_{m}$.}
        \ENSURE Recovered matrix $X$.
    \end{algorithmic}
\end{algorithm}

Accordingly, we can focus on how to optimize $\min_{X} {\rm tr}(X D X^T)$ 
subject to $X \odot P = M$.
The Lagrangian function can be represented as 
\begin{equation}
    \mathcal L = {\rm tr}(X D X^T) + {\rm tr}(V^T (X \odot P - M)),
\end{equation}
where $V \in \mathbb R^{m \times n}$ denotes Lagrange multipliers. Note that only $|\Omega|$ multipliers are needed due to the fact that $X_{ij} P_{ij} = M_{ij}$ always holds if $(i, j) \notin \Omega$. The KKT conditions are
\begin{equation}
    \left \{
    \begin{array}{l}
        \nabla_{X} \mathcal{L}(X, V) = 2X D + V \odot P = 0, \\
        X \odot P = M .
    \end{array}
    \right .
\end{equation}
To obtain the closed-form solution, we replace $D$ with 
$\hat D = D + \delta I$ ($\delta > 0$) since $\hat D$ is invertible. 
Note that $\bm x^i \odot \bm p^i = \bm x^i {\rm diag}(\bm p^i)$. 
To keep notations uncluttered, let $P_i = {\rm diag}(\bm p^i)$. 
Accordingly, we have
\begin{equation} \label{eq_lagrangian}
    \left \{
    \begin{array}{l}
        \bm x^i = -\frac{1}{2} \bm v^i P_i \hat D^{-1} \\
        \bm x^i P_i = \bm m^i
    \end{array}
    \right .
    \Rightarrow
    -\frac{1}{2} \bm v^i P_i \hat D^{-1} P_i = \bm m^i .
\end{equation}
Let $F_i = P_i \hat D^{-1} P_i$. Lemma \ref{lemma_principal_minor} shows that $[F_i]_{\bm p^i, \bm p^i}$ is invertible.

\begin{myLemma} \label{lemma_principal_minor}
    For any $Q \in \mathbb{S}_{++}^n$ and any non-zero binary vector 
    $\bm p \in \{0, 1\}^n$, $[Q]_{\bm p, \bm p}$ is positive definite.
\end{myLemma}
\begin{myDefinition}
    Given a binary vector $\bm p \in \{0, 1\}^n$ and a square 
    matrix $Q \in \mathbb R^{n \times n}$, 
    suppose that $[Q]_{\bm p, \bm p}$ is invertible. 
    We define $Q^{\bm p+} \in \mathbb R^{n \times n}$ as a matrix 
    which satisfies 
    $[Q^{\bm p+}]_{\bm p, \bm p} = [Q]_{\bm p, \bm p}^{-1}$ and the 
    other entries are 0.
\end{myDefinition}
Accordingly, $V$ and $X$ can be approximately solved by 
\begin{equation} \label{solution_X_imprecise}
    \left \{
    \begin{array}{l}
        \bm v_i = -2 \bm m^i (F_i)^{\bm p^i +} \\
        \bm x^i = \bm m^i (F_i)^{\bm p^i +} \hat D^{-1} .
    \end{array}
    \right .
\end{equation}
However, the residual caused by Eq. (\ref{solution_X_imprecise}) can not be 
guaranteed to be upper-bounded when $\exists i, \hat D_{ii} = 0$. 
To address this issue, 
define $H_i \in \mathbb{R}^{n \times n}$ as a diagonal matrix such that 
$[H_i]_{ii} = \mathbbm{1}[\hat D_{ii} \neq 0]$ 
and the other diagonal entries are 1.
If the solution is modified as  
\begin{equation} \label{solution_X_noiseless}
    \left \{
    \begin{array}{l}
        \bm v_i = -2 \bm m^i (\hat F_i)^{\bm p_i +}, \\
        \bm x^i = \bm m^i (\hat F_i)^{\bm p_i +} (H_i \hat D H_i)^{\bm p_i+} ,
    \end{array}
    \right .
\end{equation}
where $\hat F_i = P_i \hat D^{\bm h_i +} P_i$ and $\bm h_i = {\rm diag}(H_i)$, 
Theorem \ref{theo_solution_X} shows that the residual between approximate 
solution and real solution is related to $\delta$.

\begin{myTheo} \label{theo_solution_X}
    Let $\hat X$ and $\hat V$ denote the approximate solutions defined in 
    Eq. (\ref{solution_X_noiseless}). 
    There exists a constant $u$, which is independent on $\delta$, 
    such that $\hat X \odot P = M$ and 
    $\|\nabla_X \mathcal{L}(\hat X, \hat V)\| \leq 2 \delta u \|M\|$.
\end{myTheo}
Therefore, with $\delta \rightarrow 0$, $\nabla_X \mathcal L \rightarrow 0$. In our experiments, we set $\delta = 10^{-6}$. To compute $D$, we need $O(m n^2)$ time. Since we have to compute the inverse of $\hat D$, $O(n^3)$ are needed to calculate $X$ at least. Recall that $m \geq n$ and thus the computational complexity is $O(m n^2)$. 
The algorithm to solve problem (\ref{obj_no_graph}) is summarized in Algorithm \ref{alg_surrogate}.
In Section \ref{section_experiment}, we can see that our method can converge within 20 iterations. Compared with other methods that require hundreds even thousands of iterations to converge, the consuming time of the proposed model is less even though the computational complexity of each iteration is $O(m n^2)$. 
Theoretically, combining with Theorem \ref{theo_solution_D} and \ref{theo_solution_X}, we have the following proposition,
\begin{myCorollary} \label{theo_convergence}
    If $\delta \rightarrow 0$,
    then Algorithm \ref{alg_surrogate} will approach the global minimum of 
    problem (\ref{obj_no_graph}). 
\end{myCorollary}

\begin{figure}[t]
    \centering
    \subfigure[Image-1] {
        \includegraphics[width=0.47\linewidth]{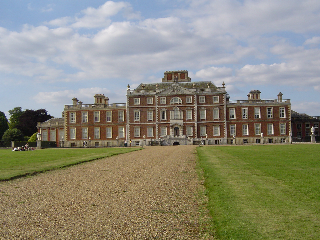}
    }
    \subfigure[Image-2] {
        \includegraphics[width=0.47\linewidth]{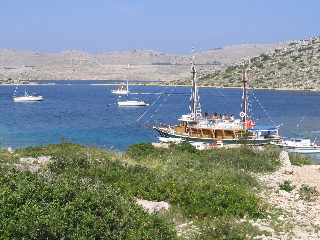}
    }
    \caption{Two images from MSRC-v2 which are used for image recovery.}
    \label{figure_datasets}
\end{figure}

\subsection{Recovery Bound}
As the recovery bound, which exposes the upper bound of errors between the recovered matrix and the real matrix, is an important part in field of matrix completion, we also provide recovery bounds about our model. Inspired by Theorem \ref{theo_solution_D}, the following theorem is the critical part to show recovery bounds. 
\begin{myTheo} \label{theo_equivalence_nuclear}
    Problem (\ref{obj_21}) is equivalent to 
    \begin{equation}
    \begin{split}
        \min \limits_{X} \|X\|_{*}, ~~ s.t. ~ X \odot P = M.
    \end{split}
    \end{equation}
\end{myTheo}
Therefore, recovery bounds for the nuclear relaxation model \cite{RecoveryBounds-2,nuclear} 
can be applied to our model. 
For instance, the famous recovery bound proposed by \cite{nuclear} is 
available for our model:
\begin{myLemma} \cite{nuclear}
    Let $X_0 \in \mathbb{R}^{m \times n}$ be the real matrix, and $N = \max(m, n)$. Suppose that $|\Omega|$ entries of $X_0$ are observed uniformly at random. Then there are constants $c_1$ and $c_2$ such that if $|\Omega| \geq c_1 N^{5/4} r \log N$, the unique minimizer equals with $X_0$ with probability at least $1 - c_2 N^{-3}\log N$. 
\end{myLemma}

\textit{Remark:} 
One may concern the significance of our relaxation since the equivalence between problem (\ref{obj_21}) and the nuclear norm surrogate. 
Roughly speaking, the main merits include the tractable optimization and scalability of our model. 
On the one hand, the optimization is completely parameter-free. Compared with gradient-based methods and ADMM-based methods, no hyper-parameters  (\textit{e.g.}, step-size in gradient-based methods, increasing coefficient in ADMM-based methods, \textit{etc.}) are required, which leads to the simple optimization. 
On the other hand, our model is well formulated from the mathematical aspect, 
since \textbf{the mathematical form is common in machine learning}. 
Therefore, it is easy to incorporate other mechanisms without obvious 
expenses and extra derivation. 
As we show in the next section, the optimization of the whole model is 
analogous to the one of problem (\ref{obj_21}) after introducing the 
correlation learning of columns. 
Contrastively, the existing models that attempt to integrate the 
additional information of matrices and the nuclear norm surrogate 
(\textit{e.g.}, LRFD \cite{LRFD}, S$^3$LR \cite{S3LR}, \textit{etc}.) 
needs a great deal of time to train. 

\subsection{Noisy Case}

Similarly, the noisy matrix completion can be modeled by adding an extra regularization, 
\begin{equation} \label{obj_no_graph_noisy}
    \begin{split}
    \min \limits_{X, D} & \overbrace{\|X \odot P - M\|_F^2 + \gamma {\rm tr}(X D X^T)}^{\mathcal J}, \\
    s.t. & ~ {\rm tr}(D^\dag) = 1, D \in \mathbb{S}_{+}^n.
    \end{split}
\end{equation}
Without any additional proofs, Theorem \ref{theo_equivalence} and \ref{theo_solution_D} can be easily extended into the noisy case. The solution of the subproblem regarding $D$ is given by Eq. (\ref{solution_D}). Since the subproblem to solve $X$ is an unconstrained problem, we can take the derivative and set it to 0,
\begin{equation}
    \nabla_X \mathcal J = 2 X \odot P - 2 M \odot P + 2 \gamma X D = 0.
\end{equation}
Similar with Eq. (\ref{solution_X_noiseless}), $X$ can be solved analytically by  adding a perturbation, 
\begin{equation} \label{solution_X_noisy}
        \bm x^i = \bm m^i P_i (P_i + \gamma \hat D)^{-1} .
\end{equation}
Motived by Theorem \ref{theo_solution_X}, $\|\nabla_X \mathcal J\| \leq 2 \delta u \|M\|$ always holds if $X$ is approximately computed by Eq. (\ref{solution_X_noisy}). Obviously, the computational cost of every iteration is $O(m n^2)$. Figure \ref{figure_illustration} gives a vivid illustration of the result of the noisy model.

Interestingly, if $D \in \mathbb{S}_{++}^n$, then problem (\ref{obj_no_graph_noisy}) can be regarded as a Maximum A Posterior (\textit{MAP}) model from the probabilistic perspective. In the probabilistic model, $X$ is a random variable, and $M$ is regarded as supervised information. Hence, the objective is to solve 
\begin{equation} \label{eq_map}
    \max \limits_{X} p(X | M) \Leftrightarrow \max \limits_{X} p(M | X) \cdot P(X) .
\end{equation}
Recall that $M_{ij} = (X_{ij})_* + \varepsilon_{ij}$ if $(i, j) \in \Omega$. Suppose that $\varepsilon_{ij} \sim \mathcal{N}(0, I)$. Therefore, $p(M_{ij} | X_{ij})$ can be modeled as $\mathcal{N}(X_{ij}, I)$ if $(i, j) \in \Omega$. To be convenient, $p(M_{ij} | X_{ij}) = 1$ if $(i, j) \notin \Omega$. We further assume that the prior distribution of $X$ is a matrix Gaussian distribution, \textit{i.e.}, $p(X) = \mathcal{MN}(0, I, \frac{1}{2 \gamma} D^{-1})$ \footnote{
$\mathcal{MN}(E, \Sigma_1, \Sigma_2) = \frac{\exp(-\frac{1}{2} {\rm tr}({\Sigma_2}^{-1} (X - E)^T {\Sigma_1}^{-1} (X - E)))}{(2 \pi)^{mn / 2} |{\Sigma_1}|^{n/2} |{\Sigma_2}|^{m/2}}$ where $E \in \mathbb{R}^{m \times n}$ is the mean, $\Sigma_1 \in \mathbb{R}^{m \times m}$ and $\Sigma_2 \in \mathbb{R}^{n \times n}$ represent covariance matrix of row and column, respectively. 
}. 
Take the $\log$ of Eq. (\ref{eq_map}),
\begin{equation}
    \begin{split}
    & \log p(M | X) + \log(X) \\
    = & \sum \limits_{i,j} \log p(M_{ij} | X_{ij}) + \log \mathcal{MN}(0, I, \frac{1}{2\gamma}D^{-1}) \\
    = & \sum \limits_{(i,j) \in \Omega} \log \mathcal{N}(0,1) + \log \mathcal{MN}(0, I, \frac{1}{2\gamma}D^{-1}) \\
    = & -\|X \odot P - M\|_F^2 - \gamma {\rm tr}(X D X^T) - \frac{n}{2} \log |D^{-1}| + C , \\
    \end{split}
\end{equation}
where $C$ denotes the constant term. Therefore, Eq. (\ref{eq_lagrangian}) is equivalent to
\begin{equation}
    \min \limits_{X, D} \|X \odot P - M\|_F^2 + \gamma {\rm tr}(X D X^T) +  \frac{n}{2} \log |D^{-1}| .
\end{equation}
Note that $\log |D^{-1}|$ can be viewed as a penalty term such that eigenvalues of $D^{-1}$ will not be too large. To simplify the model, the constraint ${\rm tr}(D^{-1})$ is used to replace $|D^{-1}|$, which can restrict eigenvalues as well. Specifically speaking, if ${\rm tr}(D^{-1}) = 1$, $|D^{-1}| < {\rm tr}(D^{-1}) = 1$ always holds. In sum, Eq. (\ref{obj_no_graph_noisy}) is thus derived from the probabilistic aspect.


\begin{algorithm}[t]
    \caption{Algorithm to solve NCARL (defined in Eq. (\ref{obj})).}
    \label{alg}
    \begin{algorithmic}
        \REQUIRE The tradeoff hyper-parameter $\alpha$, sparsity $k$, mask matrix $P$, observed entries $M$, perturbation coefficient $\delta = 10^{-6}$, and maximum iterations $t_{m} = 50$.
        \STATE $X \leftarrow M$.
        \REPEAT
            \STATE $l_{ij} \leftarrow \|\bm x_i - \bm x_j\|_2^2$.
            \STATE Update $S$ by Eq. (\ref{solution_S}).
            \STATE $S \leftarrow \frac{S + S^T}{2}$.
            \STATE Update $D$ by Eq. (\ref{solution_D}).
            \STATE $L \leftarrow D_S - S$.
            \STATE $\hat Q \leftarrow D + \alpha L + \delta I$.
            \STATE $F_i \leftarrow P_i \hat Q^{-1} P_i$.
            \STATE Update $X$ according to Eq. (\ref{solution_X_noiseless}): $\bm x^i \leftarrow \bm m^i (F_i)^{\bm p^i +} \hat Q^{-1}$.
        \UNTIL{convergence or exceeding maximum iteration $t_{m}$.}
        \ENSURE Recovered matrix $X$.
    \end{algorithmic}
\end{algorithm}

\section{Adaptive correlation Learning: Completion  correlation of Columns} 
Although the matrix completion is usually formulated as a brief optimization problem, 
the concise formulation may hide some important properties in practice. 
In this section, we will design a rational mechanism to utilize some 
additional information to improve the performance of matrix completion. 
Meanwhile, it also verifies the compatibility of the surrogate proposed in 
the previous section.

In practical applications, columns probably have underlying connections 
with each other. 
For instance, in recommendation systems, a column vector may represent 
the preferences of a user to diverse items. According to the obtained 
information (\textit{i.e.,} the observed entries), 
we can judge whether two users are similar. 
Therefore, the two recovered user vectors, which are highly similar, 
should be more analogous. 
Inspired by this, we have the following assumption,  
\begin{myAssumption} \label{assumption}
    Two vectors are similar if the Euclidean distance between them is small. 
    Formally, given $\bm x_i$, $\bm x_j$, and $\bm x_k$, 
    $\bm x_i$ are more similar with $\bm x_j$ compared with 
    $\bm x_k$ if $\|\bm x_i - \bm x_j\|_2 < \|\bm x_i - \bm x_k\|_2$. 
\end{myAssumption}

 Suppose that we have obtained similarities of some pairs of column vectors as the prior knowledge. Formally, let $S_{ij} \geq 0$ denote the similarity of $(\bm x_i)_*$ and $(\bm x_j)_*$ where $(\bm x_i)_*$ is the $i$-th column of the optimal matrix $X_*$. Clearly, $S$ should be symmetric. Naturally, the recovered matrix $X$ should keep these similarities. More formally, 
\begin{equation}
    \begin{split}
    & \min \limits_{X} \sum \limits_{i, j = 1}^n S_{ij} \|\bm x_i - \bm x_j\|_2^2 
    \Leftrightarrow \min \limits_{X} {\rm tr}(X L X^T) ,
    \end{split}
\end{equation}
where $L = D_S - S$ and $D_S$ is a diagonal matrix where $(D_S)_{ii} = \sum_{j=1}^n S_{ij}$. The matrix, $L$, is usually called Laplacian matrix in spectral graph theory \cite{spectral_graph_theory}. Hence, the noiseless model is formulated as
\begin{equation}
    \begin{split}
    \min \limits_{X, D} &~ {\rm tr}(X D X^T) + \alpha {\rm tr}(X L X^T), \\
    s.t. &~ X \odot P = M, {\rm tr}(D^\dag) = 1, D \in \mathbb{S}_+^n .
    \end{split}
\end{equation}
However, the similarity matrix $S$ is frequently unavailable in most matrix completion scenarios. Inspired by self-supervised learning \cite{self-supervised}, we compute $S$  adaptively in the training phase. Suppose that the recovered matrix is $X^{(t)}$ at step $t$. Then, the similarity $S$ is updated according to $X^{(t)}$. Accordingly, the key is how to  obtain a rational similarity matrix $S$ based on Assumption \ref{assumption}. Intuitively, for every column vector, the number of correlated columns should not be too large. In other words, $\bm s^i$ should be sparse. Besides, we normalize $S$ such that $S \textbf{1} = \textbf{1}$. In this paper, we design a novel point-wise similarity learning model, which can precisely control the sparsity degree. 
The designed correlation learning model is 
\begin{equation} \label{obj_graph}
    \begin{split}
    & \min \limits_{S \textbf{1} = \textbf{1}, S \geq 0} \sum \limits_{i, j = 1}^n S_{ij} \|\bm x_i^{(t)} - \bm x_j^{(t)}\|_2^2 +  \|\bm \mu^T S\|_F^2 \\
    \Leftrightarrow & \min \limits_{S \textbf{1} = \textbf{1}, S \geq 0} {\rm tr}(X^{(t)} L X^{(t)T}) + \|\bm \mu^T S\|_F^2 ,
    \end{split}
\end{equation}
where $\bm \mu \in \mathbb R^{n}$ represents hyper-parameters for $n$ columns. Theorem \ref{theo_graph} states that the $n$ hyper-parameters can be converted into one parameter if we assume that the sparsity degrees of columns are identical. According to Theorem \ref{theo_graph}, Corollary \ref{corollary_invariant} demonstrates that the correlation learning mechanism is scaling-invariant. 
\begin{myTheo} \label{theo_graph}
    In problem (\ref{obj_graph}), each row of the optimal $S$ is $k$-sparse (\textit{i.e.}, $\forall i, \|\bm s^i\|_0 = k$) if $\bm \mu$ satisfies
    \begin{equation}
        \frac{1}{2} (k \bm l_{i}^{(k)} - \sum \limits_{v=1}^k \bm l_{i}^{(v)} ) < \mu_i^2 \leq \frac{1}{2} (k \bm l_{i}^{(k+1)} - \sum \limits_{v=1}^k \bm l_{i}^{(v)} ), 
    \end{equation}
    where $i = 1, 2, \cdots, n$, $l_{ij} = \|\bm x_i - \bm x_j\|_2^2$ and $\bm l_i^{(k)}$ is the $k$-th smallest value in $\{l_{ij}\}_{j=1}^n$. Moreover, if $\mu_i^2 = \frac{1}{2} (k \bm l_{i}^{(k+1)} - \sum \limits_{v=1}^k \bm l_{i}^{(v)})$, $S$ can be solved by 
    \begin{equation} \label{solution_S}
        S_{ij} = (\frac{\bm l_i^{(k+1)} - l_{ij}}{\sum \limits_{v=1}^k \bm l_i^{(k+1)} - \bm l_i^{(v)}})_+ .
    \end{equation}
\end{myTheo}

\begin{algorithm}[t]
    \caption{Algorithm to solve NCARL-noisy.}
    \label{alg_noisy}
    \begin{algorithmic}
        \REQUIRE The tradeoff hyper-parameter $\alpha$ and $\gamma$, 
                    sparsity $k$, 
                    mask matrix $P$, observed entries $M$, 
                    perturbation coefficient $\delta = 10^{-6}$, 
                    and maximum iterations $t_{m} = 50$.
        \STATE $X \leftarrow M$.
        \REPEAT
            \STATE $l_{ij} \leftarrow \|\bm x_i - \bm x_j\|_2^2$.
            \STATE Update $S$ by Eq. (\ref{solution_S}).
            \STATE $S \leftarrow \frac{S + S^T}{2}$.
            \STATE Update $D$ by Eq. (\ref{solution_D}).
            \STATE $L \leftarrow D_S - S$. 
            \STATE $\hat Q \leftarrow P + D + \delta I + \alpha L$.
            \STATE Update $X$ according to Eq. (\ref{solution_X_noiseless}): $\bm x^i \leftarrow \bm m^i P_i \hat Q^{-1}$.
        \UNTIL{convergence or exceeding maximum iteration $t_{m}$.}
        \ENSURE Recovered matrix $X$.
    \end{algorithmic}
\end{algorithm}

\begin{myCorollary} \label{corollary_invariant}
    The adaptive graph learning process is scaling-invariant. In other words, if $\hat X = k X$, then the learned similarity $\hat S$ satisfies $\hat S = S$. 
\end{myCorollary}

To keep $S$ symmetric, we set $S \leftarrow (S + S^T) / 2$.
By unifying the above correlation learning and the matrix completion model proposed in Section \ref{section_mc}, the final objective of \textit{NCARL} can be formulated as 
\begin{equation}
\begin{split} \label{obj}
    \min \limits_{X, D, S} ~ &  {\rm tr}(X (D + \alpha L) X^T) + \alpha \|\bm \mu^T S\|_F^2, \\
    s.t. ~ & X \odot P = M, {\rm tr}(D^\dag) = 1, D \in \mathbb{S}_+^n, S \textbf{1} = \textbf{1}, S \geq 0.\\
\end{split}
\end{equation}
It should be pointed out that the added correlation learning mechanism does not impact the optimization of $D$. The only difference is that $D$ should be replaced by $D + \alpha L$ in Eq. (\ref{solution_X_noiseless}) when optimizing $X$. Accordingly, the adaptive correlation learning is compatible with the non-convex surrogate such that problem (\ref{obj}) can be optimized by the alternative method as well. The entire procedure is summarized in Algorithm \ref{alg}.

Similarly, the model for noisy case, \textit{NCARL-noisy}, is given as 
\begin{equation} \label{obj_noisy}
    \begin{split}
        \min \limits_{X, D} ~ & \|X \odot P - M\|_F^2 + \gamma {\rm tr}(X D X^T) \\
        & + \alpha({\rm tr}(X L X^T) + \|\bm \mu^T S\|_F^2), \\
        s.t. ~ & {\rm tr}(D^\dag) = 1, D \in \mathbb{S}_+^n, S \textbf{1} = \textbf{1}, S \geq 0 .
    \end{split}
\end{equation}
Analogous to Eq. (\ref{solution_X_noisy}), we have $\bm x^i = \bm m^i P_i (P_i + \gamma \hat D + \alpha L)^{-1}$. 
The whole procedure to solve NCARL-noisy is similar with Algorithm \ref{alg}, 
and the concrete algorithm is stated in Algorithm \ref{alg_noisy}. 
The computational cost of updating $X$ and $D$ does not change. 
The optimization of $S$ requires $O(n^2 \log n + n^2)$ and thus, 
the time complexity of each iteration is still $O(m n^2)$.

\begin{figure*}[t]
    \centering
    \subfigure[Image-1 with random mask] {
        \includegraphics[width=0.21\linewidth]{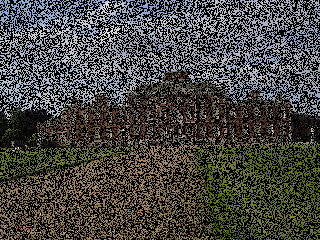}
    }
    \subfigure[Recovered Image-1] {
        \includegraphics[width=0.21\linewidth]{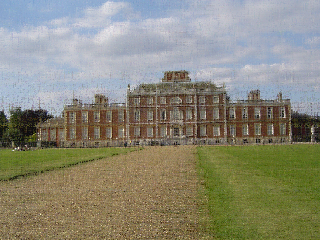}
    }
    \subfigure[Image-2 with random mask] {
        \includegraphics[width=0.21\linewidth]{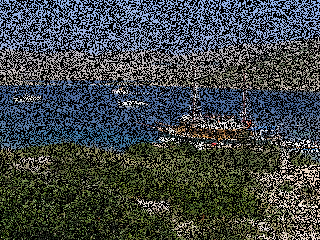}
    }
    \subfigure[Recovered Image-2] {
        \includegraphics[width=0.21\linewidth]{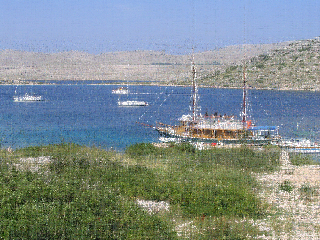}
    }

    \subfigure[Image-1 with block mask] {
        \includegraphics[width=0.21\linewidth]{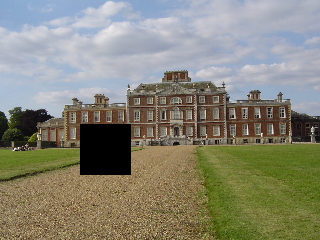}
    }
    \subfigure[Recovered Image-1] {
        \includegraphics[width=0.21\linewidth]{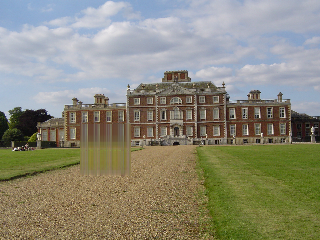}
    }
    \subfigure[Image-2 with block mask] {
        \includegraphics[width=0.21\linewidth]{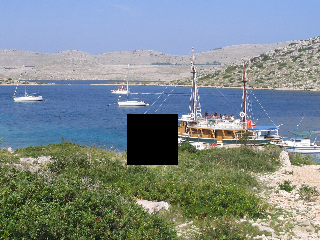}
    }
    \subfigure[Recovered Image-2] {
        \includegraphics[width=0.21\linewidth]{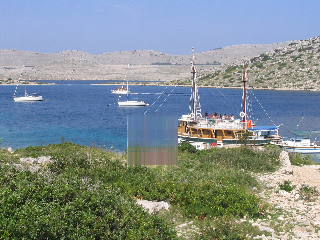}
    }

    \caption{Recovered images under different kinds of noise. 
    For random noise, the missing rate, $\epsilon$, is set as 0.5, 
    while the missing block is $50 \times 50$ for the block mask. 
    The colorful image is obtained by recovering images from three 
    channels individually.}
    \label{figure_noise}
\end{figure*}

\begin{table*}
    \centering
    \renewcommand\arraystretch{1.2}
    \setlength{\tabcolsep}{2mm}
    \caption{MSE and consuming seconds on three synthetic matrices. 
    $\epsilon$ represents the missing rate. }
    \label{table_synthetic}
    \begin{tabular}{c c c c c c c c c c c c c}
    \hline
    
    \hline
        Size & $\epsilon$ & Metric & Nuclear & F-Nuclear & $S_{1/2}$ & $S_{2/3}$ & WNNM & RegL1 & FGSR & LRFD & S$^3$LR & NCARL \\
    \hline
    \hline
        \multirow{4}{*}{$500 \times 300$} &\multirow{2}{*}{0.8} 
        & MSE & 0.0428 & 0.0515 & 0.0410 & 0.0347 & 0.0958 & 0.0473 & 0.0549 & 0.1921 & \textbf{0.0387} & \textbf{0.0344} \\
        & & Time & 3.9667 & 4.9850 & 3.7477 & \textbf{3.6688} & 18.274 & 34.134 & 9.7684 & 116.76 & 69.783 & \textbf{3.2725} \\
        & \multirow{2}{*}{0.9} & MSE & 0.0446 & 0.3484 & 0.0991 & 0.0709 & 0.0952 & 0.0492 & 0.0556 & 0.2996 & \textbf{0.0406} & \textbf{0.0397} \\
        & & Time & 3.7660 & 6.6923 & 1.6861 & \textbf{1.6662} & 11.778 & 40.323 & 7.4459 & 119.77 & 85.002 & \textbf{1.6175} \\
    \hline
        \multirow{4}{*}{$1500 \times 1000$} & \multirow{2}{*}{0.8}
        & MSE & 0.0409 & 0.1079 & 0.0433 & 0.0357 & 0.1006 & 0.0584 & \textbf{0.0377} & 0.5788 & 0.0385 & \textbf{0.0328} \\
        & & Time & 37.993 & 33.935 & 56.461 & 52.000 & 610.76 & 165.63 & \textbf{32.941} & $> 3000$ & 1612.7 & \textbf{26.262} \\
        & \multirow{2}{*}{0.9} & MSE & 0.0404 & 0.5451 & 0.0599 & 0.0450 & 0.1028 & 0.0527 & 0.0453 & 0.6978 & \textbf{0.0389} & \textbf{0.0326} \\
        & & Time & 42.730 & 36.529 & 29.242 & \textbf{28.973} & 237.84 & 126.18 & 86.099 & $> 3000$ & 1556.8 & \textbf{27.176} \\
    \hline  
        \multirow{4}{*}{$4000 \times 3000$} & \multirow{2}{*}{0.8}
        & MSE & 0.0383 & 0.1640 & 0.0445 & \textbf{0.0343} & - & - & 0.0630 & - & - & \textbf{0.0255} \\
        & & Time & 2042.9 & \textbf{374.65} & 1290.3 & 1393.1 & - & - & 713.15 & - & - & \textbf{229.50} \\
        & \multirow{2}{*}{0.9} & MSE & \textbf{0.0360} & 0.5988 & 0.0529 & 0.0408 & - & - & 0.0474 & - & - & \textbf{0.0287} \\
        & & Time & 2007.2 & 369.95 & 564.55 & \textbf{547.60} & - & - & 633.48 & - & - & \textbf{189.91} \\
    \hline
    
    \hline
    \end{tabular}
    
\end{table*}

\begin{table*}
    \centering
    \renewcommand\arraystretch{1.2}
    \setlength{\tabcolsep}{2mm}
    \caption{MSE on two noiseless datasets. \textit{Image-1} and \textit{Image-2} denote the 
    images chosen from MSRC-V2 while \textit{ML-100K} denotes the recommendation system dataset, MovieLens-100K. $\epsilon$ represents the missing rate. }
    \label{table_MSE}
    \begin{tabular}{c c c c c c c c c c c c}
    \hline
    
    \hline
         & $\epsilon$ & Nuclear & F-Nuclear & $S_{1/2}$ & $S_{2/3}$ & WNNM & RegL1 & FGSR & LRFD & S$^3$LR & NCARL \\
    \hline
    \hline
        \multirow{4}{*}{Image-1} & 0.4  & 0.1285 & 0.1285 & 0.1298 & 0.1259 & 0.1930 & 0.1285 & 0.1460 & 0.1824 & \textbf{0.1223} & \textbf{0.1267} \\
        & 0.5 & 0.1338 & 0.1338 & 0.1359 & 0.1311 & 0.1926 & 0.1338 & 0.1492 & 0.2455 & \textbf{0.1268} & \textbf{0.1309} \\
        & 0.6  & 0.1396 & 0.1396 & 0.1441 & 0.1380 & 0.1967 & 0.1396 & 0.1555 & 0.2785 & \textbf{0.1328} & \textbf{0.1385} \\
        & 0.7  & 0.1468 & 0.1468 & 0.1559 & 0.1464 & 0.2065 & 0.1468 & 0.1621 & 0.3374 & \textbf{0.1397} & \textbf{0.1451} \\
         
    \hline
        \multirow{4}{*}{Image-2} & 0.4 & 0.1880 & 0.1880 & 0.1967 & 0.1877 & 0.2911 & 0.1880 & 0.2152 & 0.2933 & \textbf{0.1786} & \textbf{0.1766} \\
        & 0.5 & 0.1913 & 0.1913 & 0.2031 & 0.1925 & 0.2847 & 0.1913 & 0.2169 & 0.7527 & \textbf{0.1813} & \textbf{0.1804} \\
        & 0.6 & 0.1975 & 0.1975 & 0.2119 & 0.2002 & 0.2908 & 0.1975 & 0.2176 & 0.8971 & \textbf{0.1886} & \textbf{0.1892} \\
        & 0.7  & 0.2068 & 0.2067 & 0.2266 & 0.2117 & 0.3003 & 0.2067 & 0.2247  & 0.9773 & \textbf{0.1973} & \textbf{0.1991} \\
    \hline  
        \multirow{4}{*}{ML-100K} & 0.4 & 0.2799 & 0.3570 & 0.3818 & 0.3342 & 0.3877 & 0.2763 & 0.3147 & 0.3962 & \textbf{0.2746} & \textbf{0.2733} \\
        & 0.5 & 0.2859 & 0.3782 & 0.4183 & 0.3650 & 0.3915 & 0.2815 & 0.3189 & 0.4219 & \textbf{0.2803} & \textbf{0.2787} \\
        & 0.6 & 0.2936 & 0.4050 & 0.4465 & 0.3918 & 0.4005 & 0.2878 & 0.3650 & 0.4463 & \textbf{0.2917} & \textbf{0.2848} \\
        & 0.7 & 0.3083 & 0.4503 & 0.4927 & 0.4386 & 0.4050 & 0.3002 & 0.3670 & 0.4511 & \textbf{0.3029} & \textbf{0.2953} \\
    \hline
    
    \hline
    \end{tabular}
    
\end{table*}

\begin{table*}[t]
    \centering
    \renewcommand\arraystretch{1.2}
    \setlength{\tabcolsep}{1.6mm}
    \caption{Consuming seconds of various models.}
    \label{table_time}
    \begin{tabular}{c c c c c c c c c c c c}
    \hline
    
    \hline
        & $\epsilon$ & Nuclear & F-Nuclear & $S_{1/2}$ & $S_{2/3}$ & WNNM & RegL1 & FGSR & LRFD & S$^3$LR & NCARL \\
    \hline
    \hline
    
        \multirow{4}{*}{IMG-1} & 0.4 & 1.5624 & \textbf{0.7689} & 4.9093 & 4.8850 & 9.4661 & 15.7755 & 1.1416 & 32.5176 & 45.3176 & \textbf{0.6179}\\
        & 0.5 & 1.5574 & \textbf{0.7913} & 4.0239 & 4.1486 & 17.7139 & 18.2746 & 1.1728 & 33.8724 & 44.3514 & \textbf{0.9211} \\
        & 0.6 & 1.6201 & \textbf{0.8880} & 3.3858 & 3.3786 & 21.0650 & 19.8961 & 1.0485 & 44.7844 & 46.4201 & \textbf{1.0525} \\
        & 0.7 & 1.5986 & \textbf{0.6103} & 2.6372 & 2.6854 & 13.8199 & 22.8290 & 0.9676 & 46.2483 & 47.9612 & \textbf{0.9640} \\
    \hline  
        \multirow{4}{*}{IMG-2} & 0.4 & 1.7186 & \textbf{1.1745} & 7.1439 & 7.5223 & 11.1063 & 15.4871 & 3.2615 & 67.1065 & 69.5478 & \textbf{1.3193} \\
        & 0.5 & 1.9493 & \textbf{1.4213} & 7.7525 & 7.5294 & 12.2682 & 17.2063 & 2.8130 & 71.0738 & 66.4808 & \textbf{1.4625} \\
        & 0.6 & 1.9046 & \textbf{1.7328} & 5.4397 & 5.5025 & 25.9807 & 18.8793 & 2.5439 & 94.2695 & 62.4828 & \textbf{1.4859} \\
        & 0.7  & 1.8639 & \textbf{1.0327} & 4.5917 & 4.3262 & 17.1553 & 21.3594 & 2.1404 & 91.5482 & 61.4658 & \textbf{1.2273} \\
    \hline
        
        \multirow{4}{*}{ML-100K} & 0.4 & \textbf{32.5527} & 38.8636 & 43.2083 & 42.8679 & 180.6862 & 1333.5766 & 53.9980 & $> 3000$ & $> 3000$ & \textbf{7.9339} \\
        & 0.5 & \textbf{34.9918} & 39.2142 & 41.6637 & 41.5215 & 144.4549 & 1538.9810 & 46.3761 & $> 3000$ & $> 3000$ & \textbf{8.0379} \\
        & 0.6 & \textbf{34.2989} & 38.8689 & 40.5688 & 40.2747 & 109.9873 & 1677.0022 & 44.2300 & $> 3000$ & $> 3000$ & \textbf{7.7124} \\
        & 0.7 & \textbf{34.7444} & 38.5751 & 38.5936 & 38.6748 & 86.8998 & 1920.2985 & 38.8642 & $> 3000$ & $> 3000$ & \textbf{8.4206} \\
    \hline
    
    \hline
    \end{tabular}
\end{table*}

\begin{figure}[t]
    \centering
    \subfigure[MSE]{
        \includegraphics[width=0.47\linewidth]{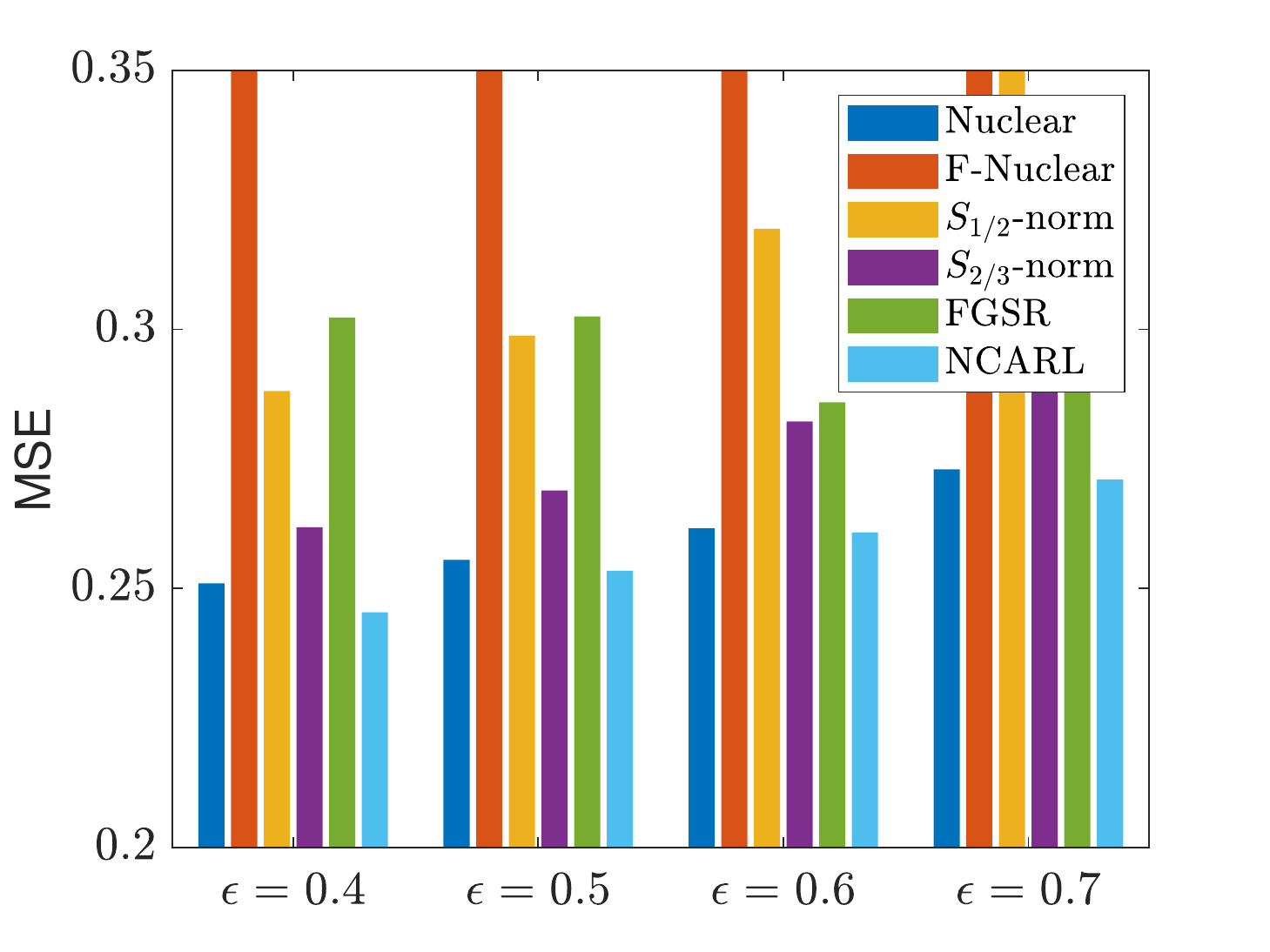}
    }
    \subfigure[Time]{
        \includegraphics[width=0.47\linewidth]{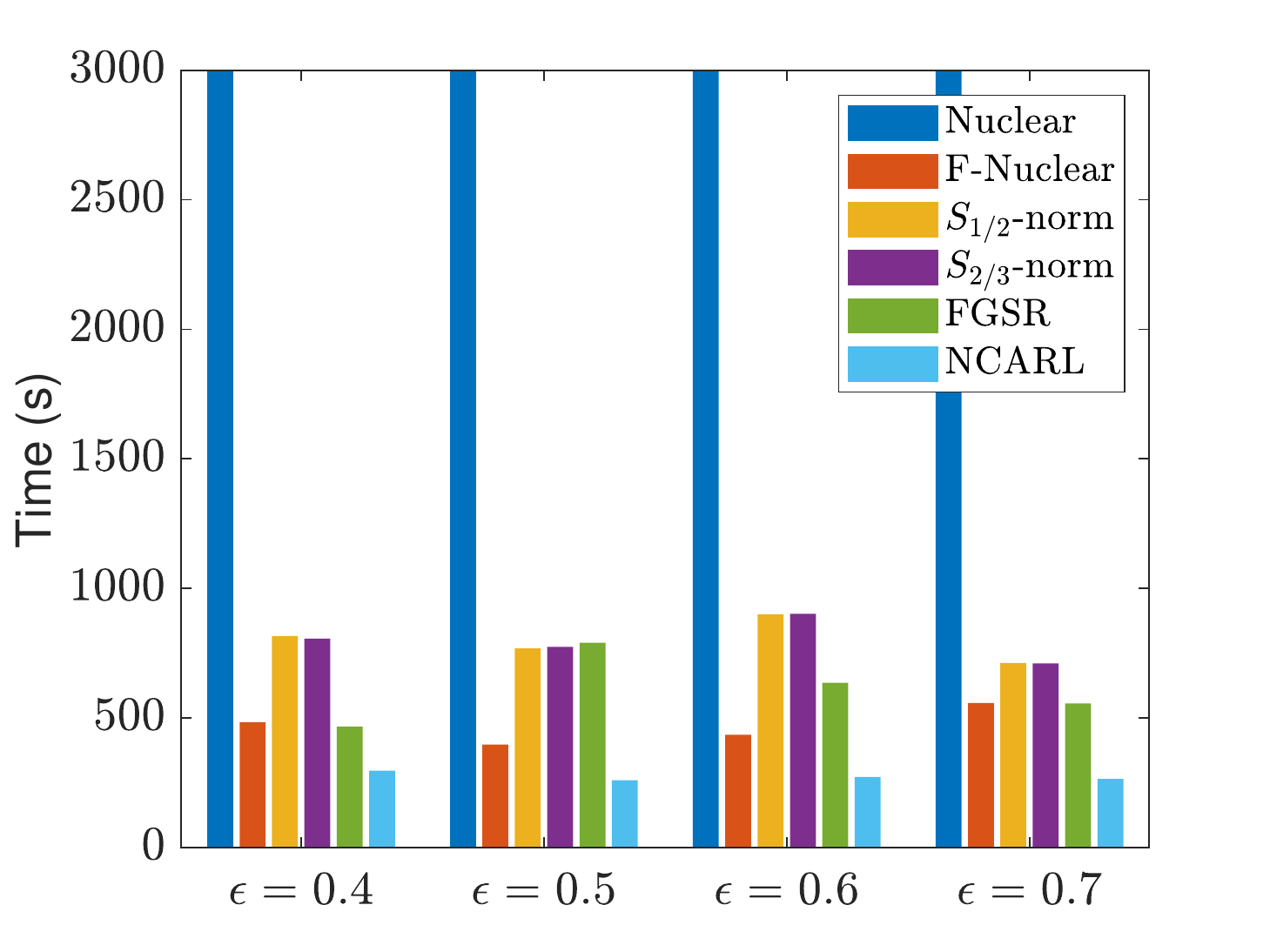}
    }
    \caption{Experiemntal results on MovieLens-1M with different $\epsilon$.
                Note that the nuclear norm consumes more than 3000s on this dataset.}
    \label{figure_movieLens_1M}
\end{figure}

\begin{figure*}[t]
    \centering
    \subfigure[Polluted Image-1] {
        \includegraphics[width=0.21\linewidth]{masked_image.png}
    }
    \subfigure[$\gamma = 10^{-6}$, $r = 240$] {
        \includegraphics[width=0.21\linewidth]{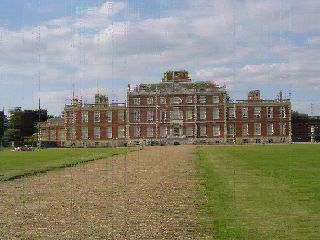}
    }
    \subfigure[$\gamma = 10^{-3}$, $r = 210$] {
        \includegraphics[width=0.21\linewidth]{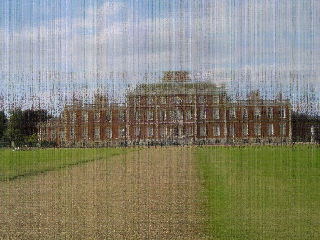}
    }
    \subfigure[$\gamma = 10^{-1}$, $r = 20$] {
        \includegraphics[width=0.21\linewidth]{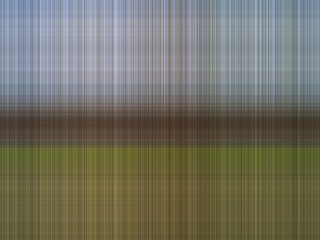}
    }

    \subfigure[Polluted Image-2] {
        \includegraphics[width=0.21\linewidth]{masked_image2.png}
    }
    \subfigure[$\gamma = 10^{-6}$, $r = 240$] {
        \includegraphics[width=0.21\linewidth]{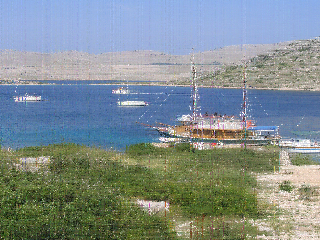}
    }
    \subfigure[$\gamma = 10^{-3}$, $r = 204$] {
        \includegraphics[width=0.21\linewidth]{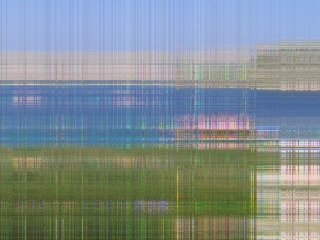}
    }
    \subfigure[$\gamma = 10^{-1}$, $r = 2$] {
        \includegraphics[width=0.21\linewidth]{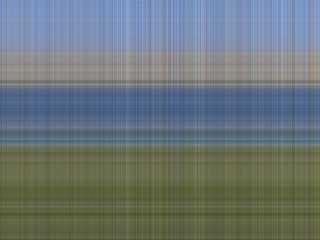}
    }
    \caption{Recovered images employing noisy model with different $\gamma$. $r$ is the average rank of recovered images from three channels. The mask rate, $\epsilon$, is set as 0.5.}
    \label{figure_illustration}
\end{figure*}

\section{Experiments} \label{section_experiment}
In this section, we test the performance of NCARL and its noisy extension on 
several real datasets. 
The experimental details, results, and analysis are reported as follows. 
All codes are implemented by MATLAB 2019b.

\subsection{Baseline Methods} 
In our experiments, 8 representative models are compared with our model, 
including classical \textit{nuclear} norm \cite{nuclear}, factored nuclear 
norm (\textit{F-Nuclear}) \cite{F-nuclear-application-1}, Schatten-$p$ norm 
($S_p$-norm) \cite{Schatten-p}, 
weighted nuclear norm (\textit{WNNM}) \cite{WNNM}, 
factored model with $\ell_1$-norm (\textit{RegL1}) \cite{RegL1}, 
factored group-sparse regularization (\textit{FGSR}) \cite{FGSR}, 
\textit{LRFD} \cite{LRFD}, and \textit{S$^3$LR} \cite{S3LR}. 
Specifically, S$^3$LR introduces the sparse subspace learning \cite{SubspaceLearning}. 
Note that the subspace learning can also be reformulated as the compatible form 
of our proposed surrogate. 
The authors \cite{S3LR} use the traditional nuclear norm and develop the optimization 
based on LADMM \cite{LADMM}. 
For Schatten-$p$ norm, we set $p$ as 1/2 and 2/3.
In particular, F-nuclear, RegL1, and FGSR are factored models, 
which are usually more efficient on large scale matrices. 
The hyper-parameters of these methods are searched in the same way 
recorded in corresponding papers.
All codes are downloaded from the authors' homepages.

\subsection{Datasets and Evaluation Metric} 
NCARL is first tested on 3 synthetic matrices with different scale, 
including $500 \times 300$, $1500 \times 1000$, and $4000 \times 3000$. 
The rank of these matrices are set as 100, 200, and 300, respectively. 
As all models performs well on the clean data when the missing entries 
is not too large, the missing rate, $\epsilon$, is set as 0.8 and 0.9. 
Moreover, we add tiny noise on 20 percent of observed entries to disturb models. 
Then, we test different models on two real images (denoted by \textit{Image-1} and \textit{Image-2}) from 
MSRC-v2 \footnote{\url{research.microsoft.com/en-us/projects/objectclassrecognition/}} \cite{msrc-v2} 
and two recommendation system datasets, MovieLens-100K and MovieLens-1M \footnote{\url{grouplens.org/datasets/movielens/}} \cite{movieLens}. 
The colorful images, which are stored as $320 \times 240 \times 3$ tensors in computer, 
are shown in Figure \ref{figure_datasets}. 
Two MovieLens datasets are two large matrices. 
Specifically, MovieLens-100K is a $943 \times 1682$ matrix
while MovieLens-1M is a $6040 \times 3952$ matrix.
For the images, we convert the image into grey pictures to obtain a matrix 
and mask $\epsilon$ ($0 < \epsilon < 1$) of pixels randomly. 
For MovieLens-100K, since only parts of entries are known, we mask $\epsilon$ of known entries randomly. 
The normalized mean-squared-error (\textit{MSE}) is employed to measure the performance of various models. 
The definition of MSE is 
\begin{equation}
    \text{MSE} = \sqrt{\frac{\sum _{(i, j) \in \Upsilon \backslash \Omega} (X_{ij} - (X_{ij})_*)^2 }{\sum _{(i, j) \in \Upsilon \backslash \Omega} (X_{ij})_{*}^{2}}} ,
\end{equation}
where $X_*$ is the true matrix and $\Upsilon$ is the set of known entries. Besides the recovery quality, the consuming time is also a vital metric in our paper.

\subsection{Experimental Setup}
There are two hyper-parameters ($\alpha$ and $k$) to tune in NCARL. 
It should be emphasized that the two hyper-parameters are introduced 
via incorporating correlation learning into the proposed surrogate that 
can be optimized via a parameter-free algorithm. 
Therefore, the hyper-parameters in NCARL do not contradict our claim about 
parameter-free optimization. 
Specifically, $\alpha$ is searched from $\{10^0, 10^1, \cdots, 10^4\}$ 
and $k$ is searched from $\{5, 10, 20, 50\}$. 
Note that the perturbation coefficient $\delta$ is set as $10^{-6}$ 
and the maximum iteration $t_{m}$ is set as 50. 
For factored models, the upper-bounds of rank are identical to each other.
On synthetic datasets, the upper-bounds are set as the exact rank. 
On two images, they are fixed as 200 while they are set as 500 on MovieLens 
datasets. 
To ensure fairness, all methods with randomness are run 5 times and 
the average results are reported. 

\begin{table*}[t]
    \centering
    \setlength{\tabcolsep}{2.5mm}
    \renewcommand\arraystretch{1.2}
    \caption{Ablation Experiments of NCARL ($\epsilon = 0.5$): \textit{Correlation} represents 
                the correlation preserving term and \textit{Adaptive} 
                denotes the adaptive learning mechanism.  }
    \label{table_ablation}
    \begin{tabular}{c c c c c c c c c c c}
        \hline
        
        \hline
         \multirow{2}{*}{}& \multirow{2}{*}{Correlation} & \multirow{2}{*}{Adaptive} & \multicolumn{2}{|c}{~Image-1~}  & \multicolumn{2}{|c}{~Image-2~} & \multicolumn{2}{|c}{~MovieLens-100K~} & \multicolumn{2}{|c}{~MovieLens-1M~} \\
         \cline{4-11}
         & & & \multicolumn{1}{|c}{MSE} & Time (s) & \multicolumn{1}{|c}{MSE} & Time (s) & \multicolumn{1}{|c}{MSE} & Time (s) & \multicolumn{1}{|c}{MSE} & Time (s) \\
         \hline
         \hline
         Method-A & \xmark & \xmark & 0.1484 & 0.8297 & 0.2030 & 1.2890 & 0.3905 & \textbf{5.9932} & 0.6481 & \textbf{236.4501} \\
         Method-B & \checkmark & \xmark & 0.1476 & 0.8098 & 0.2027 & \textbf{1.2236} & 0.2922 & 7.0448 & 0.2721 & 241.2230 \\
         NCARL & \checkmark & \checkmark & \textbf{0.1309} & \textbf{0.7302} & \textbf{0.1804} & 1.4625 & \textbf{0.2787} & 8.0379 & \textbf{0.2542} & 260.1178 \\
         \hline
         
         \hline
    \end{tabular}

\end{table*}

\subsection{Experimental Results}
\subsubsection{Synthetic Datasets} 
The results on synthetic low-rank matrices are summarized in Table \ref{table_synthetic}. 
Since WNNM, RegL1, LRFD, and S$^3$LR require too much time to converge, 
we use dash marks to represent the unavailability of these methods. 
The optimal and suboptimal results are bolded. 
From Table \ref{table_synthetic}, NCARL shows the strong stability and 
impressive efficiency due to the fast convergence. 
It should be emphasized that S$^3$LR usually provides the competitive 
results but with too much time. 
It also provides convincing evidence about the effectiveness of extra 
mechanisms for matrix completion. 
The major barrier to introduce additional information is the complicated and 
inefficient optimization. 
Therefore, the proposed surrogate, which is compatible with diverse models 
in machine learning, is meaningful. 

\subsubsection{Real Datasets}
For all real datasets, we run all methods under various missing rates
(or named as the mask rate), $\epsilon \in \{0.4, 0.5, 0.6, 0.7\}$. 
MSEs are recorded in Table \ref{table_MSE}.
For the two images and MovieLens-100K, 
the best results and second ones are highlighted in the boldface, 
while the consuming time is reported in Table \ref{table_time}. 
Due to the large scale of MovieLens-1M, several compared methods 
(Nuclear, WNNM, RegL1, LRFD, and S$^3$LR) become inefficient and thereby we 
only employ F-nuclear, Schatten-$p$ norm, and FGSR as our main competitors.
Although the classical nuclear norm model is quite time-consuming on 
MovieLens-1M, it acts as a baseline model. MSE and time on MovieLens-1M 
are shown in Figure \ref{figure_movieLens_1M}. 
From Table \ref{table_MSE}, Table \ref{table_time}, and Figure \ref{figure_movieLens_1M},
we conclude that NCARL obtains preferable performance on all datasets 
with the least time. 
As we expect, factored models like FGSR and F-Nuclear are more efficient 
especially on two MovieLens datasets.
Although the performance of S$^3$LR is usually impressive 
compared with other competitors due to the additional subspace exploration, 
its time cost is extremely expensive. 
Specifically, it needs more than 3000s to converge, which is unacceptable 
in practice, even though it returns the second-best results. 
By contrast, NCARL performs significantly in terms of both MSE and time. 

To test the recovery quality under different kinds of noise, 
images are contaminated by random noise and block noise. 
NCARL works well under both noises and 
the recovered images are shown in Figure \ref{figure_noise}.
Besides, the performance of noisy extension, NCARL-noisy, is also illustrated 
in Figure \ref{figure_illustration}.
To establish the low-rank property (indicated by Theorem \ref{theo_equivalence_nuclear}) of our surrogate meanwhile, 
we show the recovery results with different $\gamma$.
Note that the closed-form solution is approximate, 
singular values of the recovery images approach 0 numerically.
Therefore, we regard the singular values smaller than $10^{-3}$ as 0 
and mark the number of singular values larger than $10^{-3}$ as the rank $r$.
Obviously, the rank of the obtained image becomes smaller with the growth 
of $\gamma$ from Figure \ref{figure_illustration}. 

\begin{figure}[t]
    \centering
    \subfigure[Image-1]{
        \includegraphics[width=0.47\linewidth, height=3.0cm]{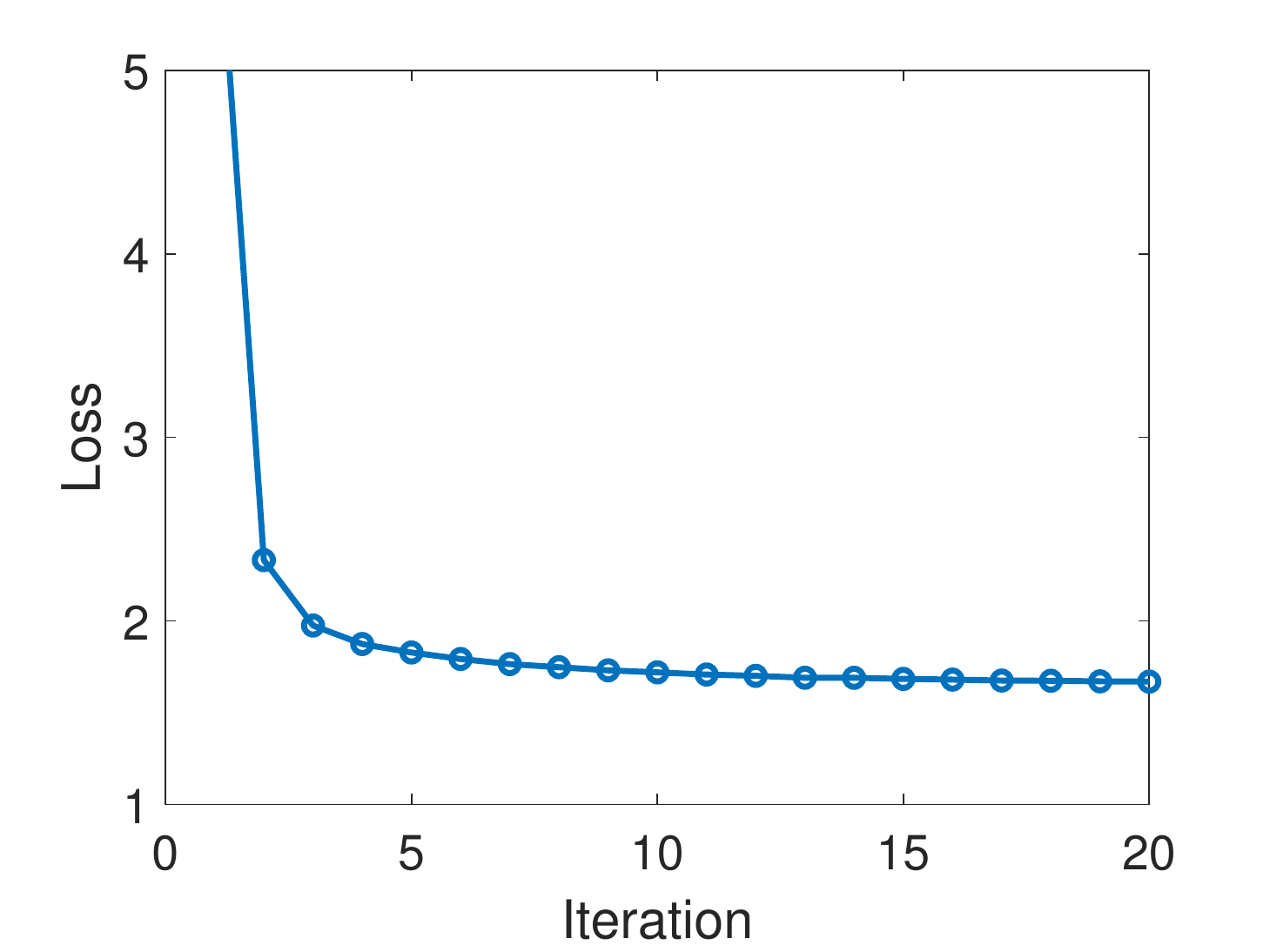}
    }
    \subfigure[MovieLens-100K]{
        \includegraphics[width=0.47\linewidth, height=3.0cm]{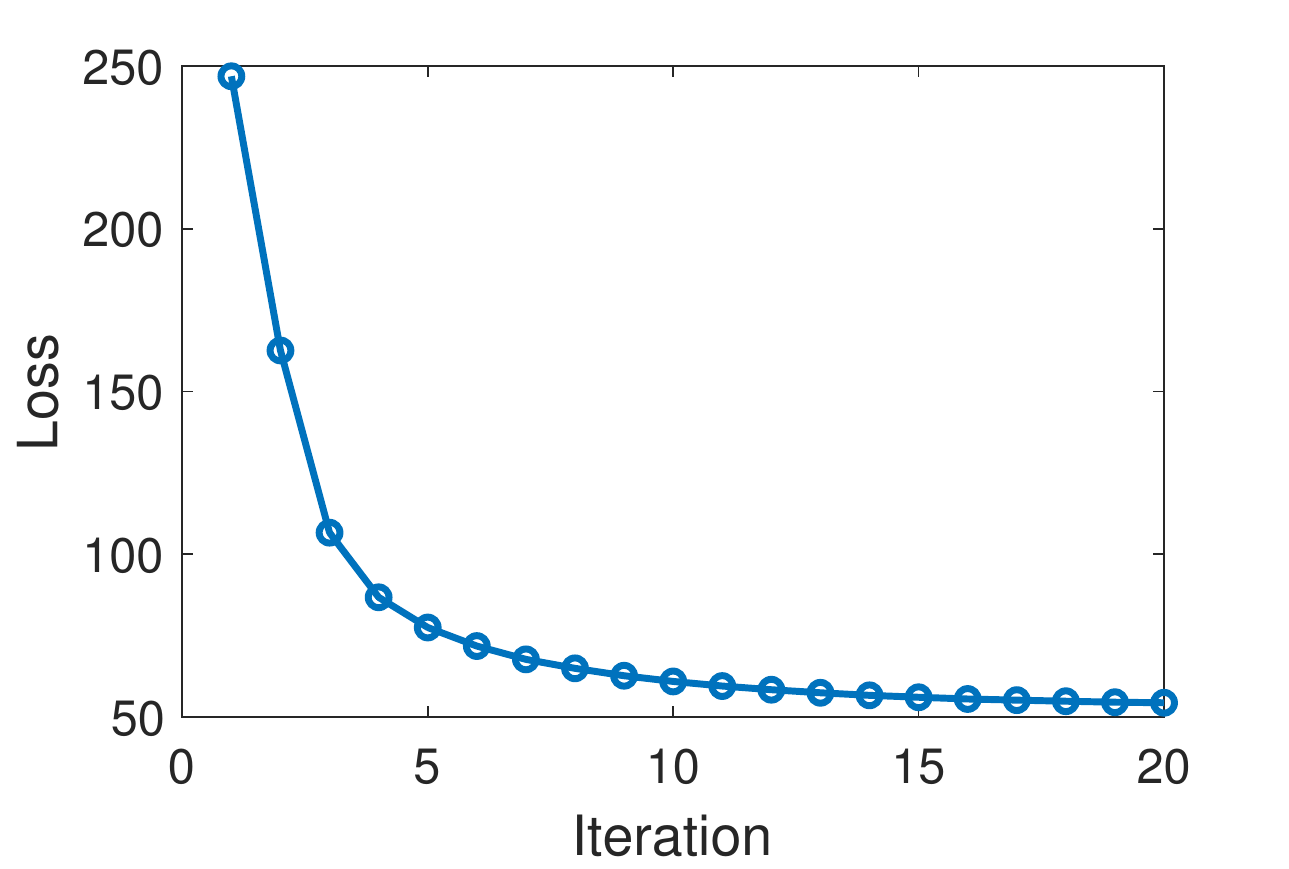}
    }
    \caption{Convergence of Algorithm \ref{alg} 
    on Image-1 and MovieLens-100K with $\epsilon=0.5$. 
    Clearly, NCARL converges rapidly on both datasets, especially Image-1}
    \label{figure_convergence}
\end{figure}

To testify the rapid convergence of NCARL, we show the objective value 
of NCARL on Image-1 and MovieLens-100K in Figure \ref{figure_convergence}.
Clearly, NCARL converges fast within 20 iterations. 
This attractive trait is more apparent on images. 
Contrasively, the other models, which are solved by gradient-based methods 
or ADMM-based methods, require hundreds or thousands of iterations to converge.
Accordingly, NCARL can be used in large scale datasets (like MovieLens-1M) 
as well, though the computational complexity of each iteration is $O(mn^2)$.

To study the effect of sparsity $k$, results with different $k$ are shown 
in Figure \ref{figure_sensitivity}, where $\alpha$ is assigned as the 
best value from $\{10^0, 10^1, \cdots, 10^4\}$. 
In our experiments, $\alpha$ is set as 100 on Image-1 and $10^2$ on MovieLens-100K. 
From Figure \ref{figure_sensitivity}, 
it is not hard to find that the proposed model will get the best 
performance when $k$ is not too large and the optimal $k$ becomes larger 
with the increase of matrix scale. 
On Image-1, the best sparsity is $10$ 
while $k = 100$ will lead to the best MSE on MovieLens-100K. 
Moreover, the increase of $k$ will not burden the time cost obviously.

\begin{figure}[t]
    \centering
    \subfigure[Image-1]{
        \includegraphics[width=0.47\linewidth]{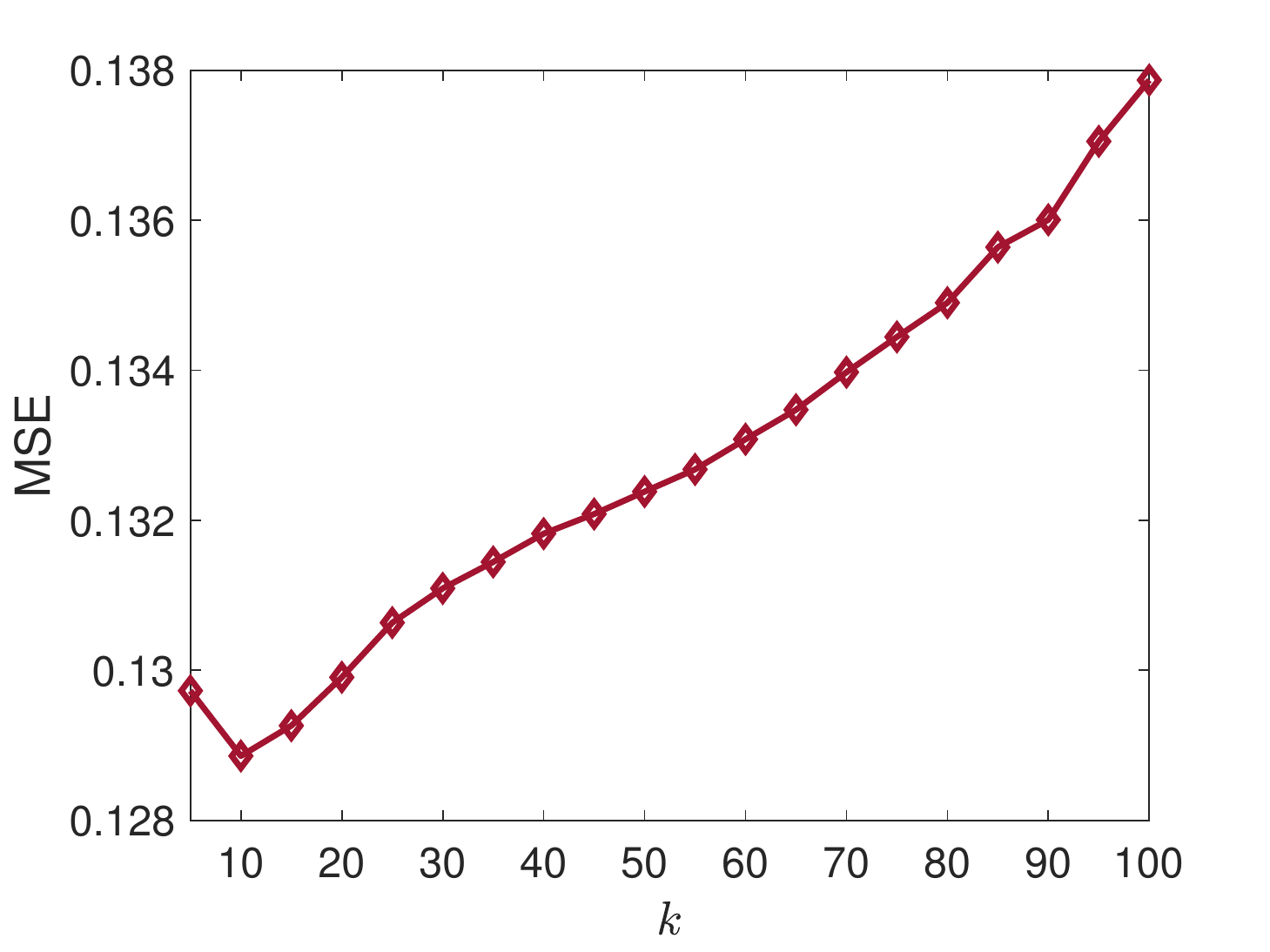}
    }
    \subfigure[MovieLens-100K]{
        \includegraphics[width=0.47\linewidth]{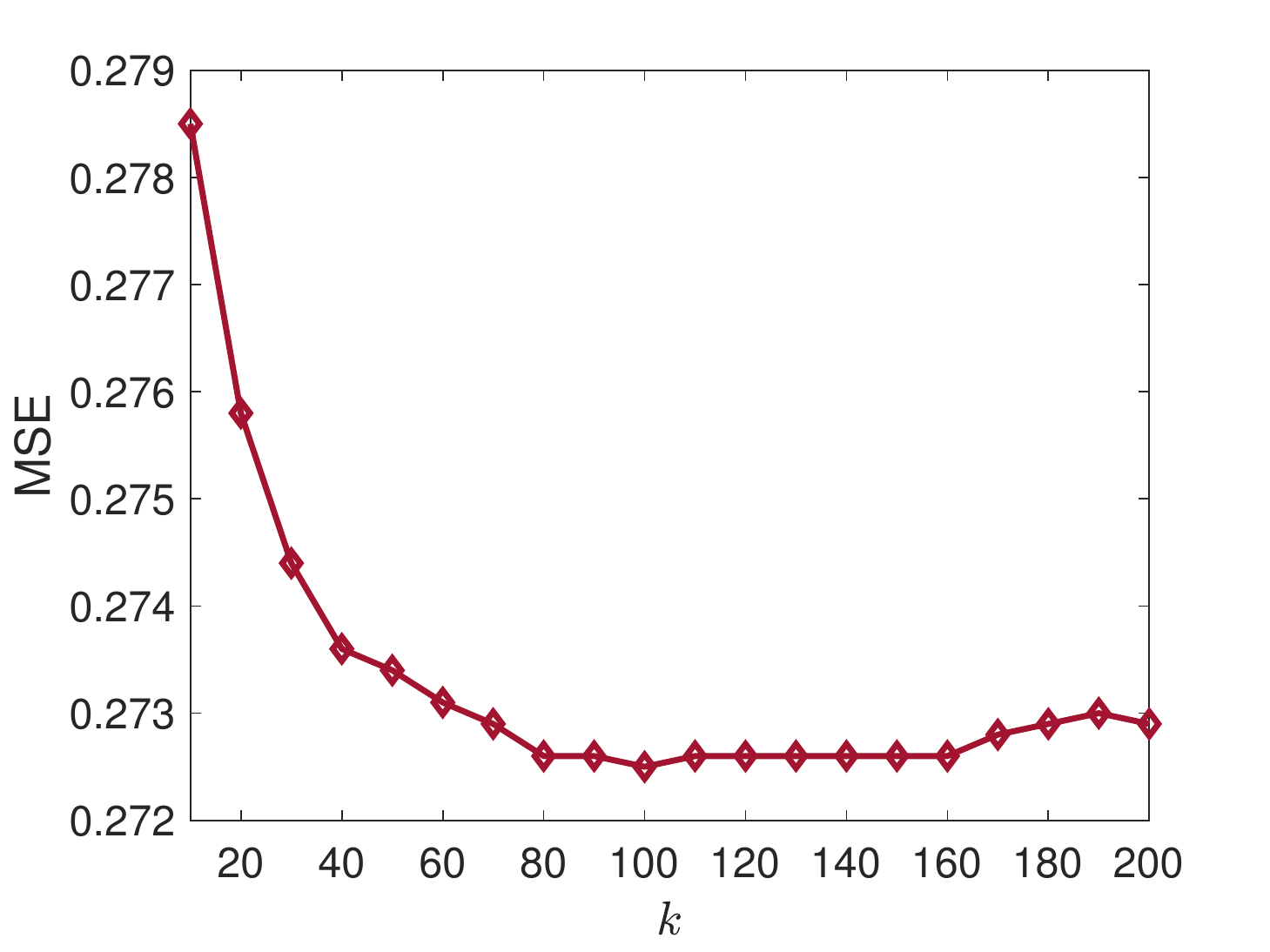}
    }
    \caption{The influence of sparsity $k$ to MSE
    on Image-1 and MovieLens-100K with $\epsilon=0.5$.}
    \label{figure_sensitivity}
\end{figure}

\subsection{Ablation Analysis} 
To test the impact of different parts, 
we design ablation experiments on all datasets. 
There are totally 2 mechanisms to testify, 
column correlation preserving term and adaptive correlation learning mechanism. 
Accordingly, we conduct experiments to study the role of two parts 
and the results are recorded in Table \ref{table_ablation}. 
On the one hand, we conclude that the correlation preserving is 
useful for LRMC, especially on recommendation system datasets. 
Specifically speaking, the correlation preserving decreases MSE by about 0.1 
and 0.37 on MovieLens-100K and MovieLens-1M, respectively.
On the other hand, the adaptive learning mechanism further promotes the 
performance of our model. Compared with the method only with the 
correlation preserving, the adaptive learning reduces MSE by about 0.2 on 
two MovieLens datasets. 

In particular, the two mechanisms do not burden the time cost significantly.
In contrast, S$^3$LR, which introduces sparse subspace learning into the 
nuclear norm model, requires much more time to train than the original nuclear 
model.

\section{Conclusion}
In this paper, we propose a novel model for low-rank matrix completion. 
Rather than the nuclear norm, a non-convex surrogate is developed. 
Although the surrogate is non-convex, it is easy to optimize and extend 
since the optimization consists of multiple closed-form solutions.
Based on the proposed relaxation, we introduce an adaptive correlation 
learning to explore the underlying information of the matrix, 
which is inspired by recommendation systems. 
Although the computational complexity of each iteration is $O(m n^2)$, 
the algorithm converges so fast that it needs less time than the existing 
methods. 
We conduct experiments on 2 real images and 2 recommendation system datasets 
and the superiority of our model is supported on both recovery quality and 
consuming time. 
In the future work, we will focus on the investigation about the convergence 
rate since the rapid convergence is only verified empirically.





    

\newpage

~\\

\newpage 

\begin{appendices}

\section{Proofs} \label{section_proof}
In this part, proofs of the above theorems and propositions are 
elaborated successively.

\subsection{Proof of Theorem \ref{theo_equivalence_20_rank}}
\begin{proof}
    Let 
    \begin{equation}
        \left \{
        \begin{array}{l}
            \mathcal{J}_1 = \min \limits_{U, W} \|W\|_{2,0}, ~~ s.t. ~ (W, U) \in \Psi', \\
            \mathcal{J}_2 = \min \limits_{X} {\rm rank}(X), ~~ s.t. ~ X \odot P = M .
        \end{array}
        \right .
    \end{equation}
    On the one hand, for the optimal $(W_*, U_*)$, we can easily find an 
    $X = (U_* W_*)^T \in \{X | X \odot P = M\}$, which indicates $\mathcal{J}_2 \leq \mathcal{J}_1$.

    On the other hand, for the optimal $X_*$, we can apply the full-rank 
    factorization on $X_*$, and thus we can get $\mathcal{J}_1 \leq \mathcal{J}_2$.

    Overall, we have $\mathcal{J}_1 = \mathcal{J}_2$.
\end{proof}

\begin{figure}
    \centering
    \includegraphics[width=0.9\linewidth]{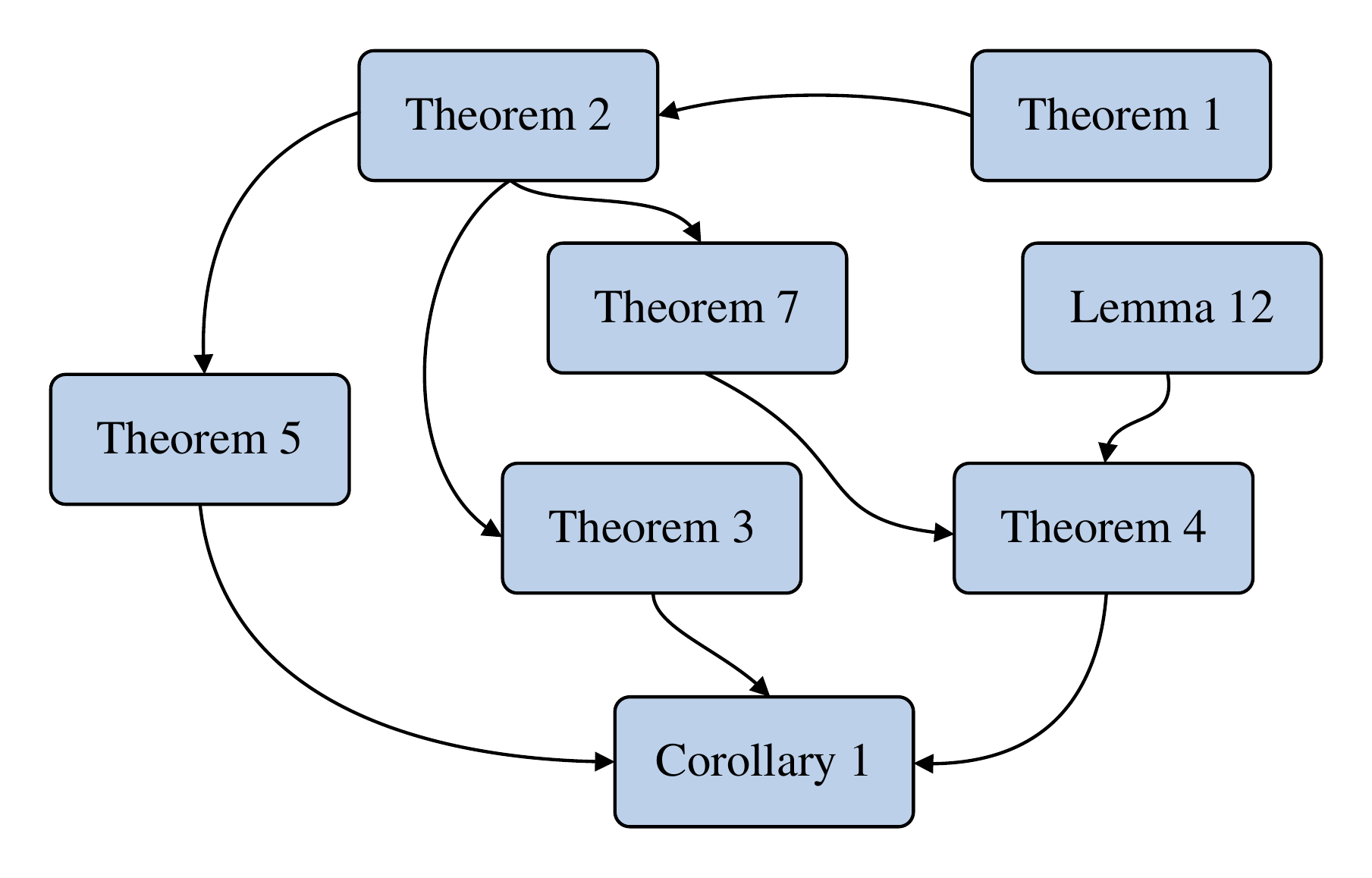}
    \caption{Proof sketch of the main theoretical results.}
    \label{figure_proofs}
\end{figure}

\subsection{Proof of Theorem \ref{theo_equivalence}}

The following theorem demonstrates that the above objective function can be converted into a smooth function that has a continuous first-order derivative. 

\begin{myLemma} \label{lemma}
    Given $k$ non-negative constants $c_i$ and variable $\bm x \in \{\bm x | \bm x^T \textbf{1} = 1, \bm x > 0\}$, the following inequality holds
    \begin{equation}
        \sum \limits_{i=1}^k \frac{c_i^2}{x_i} \geq (\sum \limits_{i=1}^k c_i)^2 .
    \end{equation}
    The equality holds if and only if $x_i = \frac{c_i}{\sum_{i=1}^k c_i}$.
\end{myLemma}

\begin{proof}
    Let $f(\bm x) = \sum_{i=1}^k \frac{c_i^2}{x_i}$ where $\bm x > 0$. At first, we will show that $f(\bm x)$ is convex. Clearly, 
    \begin{equation}
        \nabla f(\bm x) = -
        \left [
        \begin{array}{c}
            \frac{c_1^2}{x_i^2} \\
            \vdots \\
            \frac{c_k^2}{x_k^2}
        \end{array}
        \right ], 
        \nabla^2 f(\bm x) = 2 
        \left [
        \begin{array}{c c c}
            \frac{c_1^2}{x_1^3} & & \\
            & \ddots & \\
            & & \frac{c_k^2}{x_k^3}
        \end{array}
        \right ] \in \mathbb{S}_{++}^k .
    \end{equation}
    The convex property indicates that the objective is equivalent to prove $(\sum_{i=1}^k c_i)^2$ is the infimum of $f(\bm x)$ for $\bm x \in \{\bm x | \bm x^T \textbf{1} = 1, \bm x > 0\}$. Now, we solve 
    \begin{equation}
        \min \limits_{\bm x^T \textbf{1} = 1, \bm x \geq 0} \sum \limits_{i=1}^k \frac{c_i^2}{x_i} ,
    \end{equation}
    via Lagrangian method. Let $\lambda$ and $\bm \eta \geq 0$ be Lagrangian multipliers, 
    \begin{equation}
        \mathcal L = \sum \limits_{i=1}^k \frac{c_i^2}{x_i} + \lambda (\sum \limits_{i=1}^k x_i - 1) - \bm \eta^T \bm x .
    \end{equation}
    Therefore, the KKT conditions can be formulated as 
    \begin{equation}
        \left \{
        \begin{array}{l}
            -\frac{c_i^2}{x_i^2} + \lambda - \eta_i = 0, \\
            \sum \limits_{i=1}^k x_i = 1, \\
            \eta_i x_i = 0,
        \end{array}
        \right .
        \Rightarrow
        \left \{
        \begin{array}{l}
            \eta_i = 0 ,\\
            \lambda = (\sum \limits_{i=1}^k c_i)^2 ,\\
            x_i = \frac{c_i}{\sum \limits_{i=1}^k c_i} .
        \end{array}
        \right .
    \end{equation}
    Substitute the solution into $f(x)$ and we get $f_*(x) = (\sum_{i=1}^k c_i)^2$. Hence, the lemma is proved. 
\end{proof}

\begin{proof} [Proof of Theorem \ref{theo_equivalence_20_rank}]
    To keep simplicity, let 
    \begin{equation}
        \left \{
        \begin{array}{l}
            \mathcal J_0 = \|(W^T U^T) \odot P - M\|_F^2 + \gamma \|W\|_{2,1}^2 ,\\
            \mathcal J_1 = \|X \odot P - M\|_F^2 + \gamma {\rm tr}(X D X^T) .
        \end{array}
        \right .
    \end{equation}
    On the one hand, for any $W$ and orthogonal matrix $U$, 
    we can construct $X = W^T U^T$, $D = U \Lambda U^T$, and 
    $\Lambda = {\rm diag}(\frac{\|\bm w^i\|_2}{\|W\|_{2,1}})^\dag$. Then, 
    \begin{equation}
        \begin{split}
            & {\rm tr}(X D X^T) = {\rm tr}(X U \Lambda U^T X^T) \\
            = & {\rm tr}(W^T \Lambda W) = {\rm tr}(\sum \limits_{i} \Lambda_{ii} \bm w^{iT} \bm w^i) \\
            = & \sum \limits_{i} {\rm tr}(\Lambda_{ii} \|\bm w^i\|_2^2) = \|W\|_{2,1}^2 .
        \end{split}
    \end{equation}
    Hence, we have $\min \mathcal J_1 \leq \min \mathcal J_0$. 
    On the other hand, for any $X$ and $D$, 
    we can perform full rank factorization to $X$ and eigenvalue factorization to $D$, 
    which means $X = W^T U^T$ and $D = U \Lambda U^T$. Similarly, 
    \begin{equation}
        {\rm tr}(X D X^T) = \sum \limits_{i} {\rm tr}(\Lambda_{ii} \|\bm w^i\|_2^2) .
    \end{equation}
    According to Lemma \ref{lemma}, we have 
    ${\rm tr}(X D X^T) \geq (\sum_i \|\bm w^i\|_2)^2 = \|W\|_{2,1}^2$. 
    The equality holds if and only if $\Lambda_{ii} = (\frac{\|\bm w^i\|_2}{\|W\|_{2,1}})^\dag$. 
    Hence, $\min \mathcal J_1 \geq \min \mathcal J_0$.
    
    In sum, the theorem is proved.
\end{proof}

\subsection{Proof of Proposition \ref{proposition}}
\begin{proof}
    We can consider the special case that $X = x \in \mathbb R$, 
    $D = y \in \mathbb R$, and $P = M = 0$. Accordingly, 
    \begin{equation}
        \begin{split}
        & \nabla {\rm tr}(X D X^T) = 
        \left [
        \begin{array}{c}
            2 x y \\
            x^2
        \end{array}
        \right ] \\
        \Rightarrow
        & \nabla^2 {\rm tr}(X D X^T) = 
        \left [
        \begin{array}{c c}
            2 y & 2x\\
            2 x & 0
        \end{array}
        \right ]
        = H .
        \end{split}
    \end{equation}
    When $x = y = 1$ and $\bm v = [1; -1]$, 
    we have $\bm v^T H \bm v = -1 < 0$. 
    Hence, the problem is non-convex.
    Note that the constraint ${\rm tr}(D^\dag) = 1$ is not 
    convex such that the problem regarding $D$ is non-convex. 
    When $D$ is fixed, the subproblem, 
    \begin{equation}
        \min \limits_{X} {\rm tr}(X D X^T) , ~~~~ s.t. ~ X \odot P = M,
    \end{equation}
    is convex. To show it, we rewrite ${\rm tr}(X D X^T)$  as
    \begin{equation}
        {\rm tr}(X D X^T) = \sum \limits_{i} \bm x^i D (\bm x^i)^T .
    \end{equation}
    Expand $X$ according to rows, and we have 
    \begin{equation}
        \begin{split}
        & \nabla {\rm tr}(X D X^T) = 
        \left [
        \begin{array}{c}
            D \\
            D \\
            \vdots \\
            D
        \end{array}
        \right ] \\
        \Rightarrow
        & \nabla^2 {\rm tr}(X D X^T) = 
        \left [
        \begin{array}{c c c c}
            D & 0 & \cdots & 0 \\
            0 & D & \cdots & 0 \\
            \vdots & \vdots & \ddots & \vdots \\
            0 & 0 & \cdots & D \\
        \end{array}
        \right ] \in \mathbb{S}_+^{m n} .
        \end{split}
    \end{equation}
    Hence, the proposition is proved.
\end{proof}

\subsection{Proof of Theorem \ref{theo_solution_D}}

\begin{proof}
    According to the Cauchy-Schwarz inequality, we have 
    \begin{equation}
    \begin{split}
        & {\rm tr}(X D X^T) \\
        = & {\rm tr}(D^{\frac{1}{2}} (X^T X)^{\frac{1}{2}} (D^{\frac{1}{2}} (X^T X)^{\frac{1}{2}})^T) \cdot {\rm tr}((D^\dag)^{\frac{1}{2}} (D^\dag)^{\frac{1}{2}}) \\
        \geq & {\rm tr}(D^{\frac{1}{2}} (X^T X)^{\frac{1}{2}} (D^\dag)^{\frac{1}{2}})^2  .
    \end{split}
    \end{equation} 
    The equality holds if and only if 
    \begin{equation} \label{cauchy_cond}
        D^{\frac{1}{2}} (X^T X)^{\frac{1}{2}} = k (D^\dag)^{\frac{1}{2}}  .
    \end{equation}
    Since $X = W^T U^T$ and $D = U \Lambda U^T$, we can rewrite the right hand as 
    \begin{equation}
    \begin{split}
        {\rm tr}(X D X^T) & \geq {\rm tr}(U \Lambda^{\frac{1}{2}} U^T (X^T X)^{\frac{1}{2}} U (\Lambda^\dag)^{\frac{1}{2}} U^T)^2 \\
        & = {\rm tr}(U \hat I U^T(X^T X)^{\frac{1}{2}})^2 \\
        & = (\sum \limits_{i} \mathbbm{1}[\frac{\|\bm w^i\|_2}{\|W\|_{2,1}} > 0] {\rm tr}(\bm u_i \bm u_i^T (X^T X)^{\frac{1}{2}}) )^2 .
    \end{split}
    \end{equation}
    where $\hat I = \Lambda \Lambda^\dag$. Note that ${\rm tr}(\bm u_i \bm u_i^T (X^T X)^{\frac{1}{2}} )$ can be rewritten as 
    \begin{equation}
        \begin{split}
        {\rm tr}(\bm u_i \bm u_i^T (X^T X)^{\frac{1}{2}} ) & = {\rm tr}(\bm u_i \bm u_i^T (U W W^T U^T )^{\frac{1}{2}} ) \\
        & = {\rm tr}(\bm u_i^T U (W W^T)^{\frac{1}{2}} U^T \bm u_i) \\
        & = {\rm tr}(\bm e_i^T V S^{\frac{1}{2}} V^T \bm e_i) ,
        \end{split}
    \end{equation}
    where $\bm e_i = U^T \bm u_i$ and $W W^T = V S V^T$. Note that $\bm e_i^T V S V^T \bm e_i = 0$, $\bm e_i^T V S^{\frac{1}{2}} V^T \bm e_i = 0$. Hence, we have 
    \begin{equation}
        \begin{split}
            {\rm tr}(X D X^T) & \geq (\sum \limits_{i} \mathbbm{1}[\frac{\|\bm w^i\|_2}{\|W\|_{2,1}} > 0] {\rm tr}(\bm u_i \bm u_i^T (X^T X)^{\frac{1}{2}}) )^2 \\
            & = {\rm tr}((X^T X)^\frac{1}{2})^2 = \|X\|_*^2 ,
        \end{split}
    \end{equation}
    where $\mathbbm{1}[\cdot]$ is the indicator function. Formally, $\mathbbm{1}[\cdot] = 1$ if $\cdot$ is true; otherwise, $\mathbbm{1}[\cdot] = 0$.
\end{proof}

\subsection{Proof of Lemma \ref{lemma_principal_minor}}

\begin{proof}
    Given a matrix $Q \in \mathbb{S}_{++}^{n}$ and arbitrary vector $\bm x \in \mathbb R^n$, we have 
    \begin{equation}
        \bm x^T Q \bm x > 0 .
    \end{equation}
    For any binary vector $\bm p$, suppose that $\|\bm p\|_0 = k$. Accordingly, the sub-matrix $[Q]_{\bm p, \bm p}$ is a $k \times k$ matrix. Given an arbitrary vector $\bm y \in \mathbb R^k$, we can construct a $n$-dimension vector, $\bm {v}$, such that $\bm{v}_{[\bm p]} = \bm y$ and $[\bm{v}]_{\bm{\bar p}} = 0$. Accordingly, we have that
    \begin{equation}
        0 < \bm v^T Q \bm v = \sum \limits_{i,j=1}^n v_i v_j Q_{ij} = \bm y^T [Q]_{\bm p, \bm p} \bm y .
    \end{equation}
    Hence, the theorem is proved.
\end{proof}

\subsection{Proof of Theorem \ref{theo_solution_X}}
To prove Theorem \ref{theo_solution_X} completely, we will prove it by two sub-theorems.
First, if $D_{ii} \neq 0$ for any $i$, we aim to prove the following theorem.
\begin{myTheo} \label{theo_solution_X_special}
    Let $\hat X$ and $\hat V$ denote the approximate solutions defined as
    \begin{equation} \label{solution_X_noiseless_appendix}
        \left \{
        \begin{array}{l}
            \hat{\bm \gamma}_i = -2 \bm m^i (F_i)^{\bm p_i+}, \\
            \hat{\bm x}^i = \bm m^i (F_i)^{\bm p_i+} \hat D^{-1} .
        \end{array}
        \right .
    \end{equation}
    where $F_i = P_i \hat D^{-1} P_i$, $H_i = [F_i]_{\bm p_i, \bm p_i} \in \mathbb R^{r_i \times r_i}$ and $r_i = \|\bm p_i\|_0$.
    If $\forall i, D_{ii} \neq 0$,
    then there exists a constant $u$, which is independent on $\delta$, 
    such that $\hat X \odot P = M$ and $\|\nabla_X \mathcal{L}(\hat X, \hat V)\| \leq 2 \delta u \|M\|$.
\end{myTheo}

The above theorem proves Theorem \ref{theo_solution_X} partially. 
Then, we will prove the more general case 
when there exists $i$ which satisfies $D_{ii} = 0$, 
which completes the proof.


\subsubsection{Proof of Theorem \ref{theo_solution_X_special}}

\begin{myLemma} \label{lemma_submatrix_inverse}
    Given a binary vector $\bm p \in \mathbb R^{n}$ and a square matrix $S \in \mathbb R^{n \times n}$, suppose that $[S]_{\bm p, \bm p}$ is invertible. Then we have $P S P S^{\bm p+} = P$ where $P = {\rm diag}(\bm p)$.
\end{myLemma}

\begin{proof}
    Let $A = P S P$ and $B = S^{\bm p+}$. Then we have
    \begin{equation}
        \begin{split}
            (AB)_{i j} = \sum \limits_{k} a_{i k} b_{i k} .
        \end{split}
    \end{equation}
    It is not hard to see that 
    \begin{equation}
        A B = P .
    \end{equation}
\end{proof}

\begin{myLemma}
    For any $S \in \mathbb{S}^{n}_{++} $ and $i \in \{1, 2, \cdots, n\}$, $S_{ii} \neq 0$. 
\end{myLemma}

\begin{proof}
    Use the eigenvalue decomposition, and we can factor $S$ as
    \begin{equation}
        S = U^T \Lambda U .
    \end{equation}
    Then, we have 
    \begin{equation}
        S_{ii} = \bm u_i^T \Lambda \bm u_i > 0 .
    \end{equation}
\end{proof}

\begin{myLemma} \label{lemma_norm_inequality}
    Given a matrix $A \in \mathbb R^{m \times n}$ and $\bm x \in \mathbb R^{n}$, the following inequality,
    \begin{equation}
        \|A \bm x\|_2 \leq \|A\|_F \|\bm x\|_2 .
    \end{equation}
\end{myLemma}

\begin{myLemma} \cite{submatrix} \label{lemma_sub_matrix}
    Let $A \in \mathbb R^{n \times n}$ be an invertible matrix, and $B = A^{-1}$. If $b_{qp} \neq 0$ for any $p, q \in \{1, 2, \cdots, n\}$, then $A_{\bar p, \bar q} \in \mathbb R^{(n-1) \times (n-1)}$ is invertible. Let $M = A_{\bar p, \bar q}^{-1}$, and then we have
    \begin{equation}
        m_{ij} = b_{ij} - \frac{b_{ip} b_{qj}}{b_{qp}} .
    \end{equation}
\end{myLemma}

The lemma is proved in literature \cite{submatrix}. To keep the notations uncluttered, we use $\|\cdot\|$ to replace $\|\cdot\|_F$ for short.

\begin{myLemma} \label{lemma_upper}
    For $\forall S \in \mathbb S^{n}_{++}$ and binary vector $\bm p$ which 
    satisfies $\|\bm p\|_0 = n - 1$,  
    $\|S^{\bm p+} S\|^2 < (n - 1) + \frac{\|\bm b_k\|_2^2}{b_{kk}^2}$ 
    where $ p_k = 0$ and $B = S^{-1}$.
\end{myLemma}

\begin{proof}
    Without loss of generality,  we assume that $k = n$, \textit{i.e.}, $p_n = 0$. To be simple, let $M = S^{\bm p+}$. According to Lemma \ref{lemma_sub_matrix}, we have
    \begin{equation}
        m_{ij} = 
        \left \{
        \begin{array}{l l}
            b_{ij} - \frac{b_{in} b_{nj}}{b_{nn}}, & i,j \leq n - 1 ; \\
            0, & \text{otherwise} .
        \end{array}
        \right .
    \end{equation}
    Let $T = S^{\bm p+} S$. Without formal proof, $\bm t^n = 0$ and $[T]_{\bm p, \bm p} = I$. Now, we focus on $\bm t_n$ using the above equation. For $i \neq n$, 
    \begin{equation}
        \begin{split}
            t_{i n} & = \sum \limits_{l} m_{i l} a_{l n} = \sum \limits_{l} (b_{i l} - \frac{b_{in} b_{nl}}{b_{nn}}) a_{l n} \\
            & = \sum \limits_{l} b_{i l} a_{l n} - \frac{b_{i n}}{b_{nn}} \sum \limits_l b_{n l} a_{l n} .
        \end{split}
    \end{equation}
    Due to $A B = I$, we have 
    \begin{equation}
        (AB)_{i n} = \sum \limits_{l} b_{i l} a_{l n} =
        \left \{
        \begin{array}{l l}
            1, & i = n ; \\
            0, & i \neq n .
        \end{array}
        \right .
    \end{equation}
    Accordingly, 
    \begin{equation}
        t_{in} = - \frac{b_{in}}{b_{nn}} .
    \end{equation}
    Hence, we have 
    \begin{equation}
        \|S^{\bm p+} S\|^2 = \|T\|^2 = (n - 1) + \sum \limits_{i \neq n} \frac{b_{in}^2}{b_{nn}^2} < (n - 1) + \frac{\|\bm b_{n}\|_2^2}{b_{nn}^2} .
    \end{equation}
    Hence, the lemma is proved. 
\end{proof}

\begin{myLemma} \label{lemma_assumption}
    If $\forall i, D_{ii} \neq 0$, then  
    \begin{equation}
        \frac{\|\bm {\hat d_{i}}\|_2^2}{\hat d_{ii}^2} < \frac{\|\bm d_{i}\|_2^2}{d_{ii}^2} .
    \end{equation}
\end{myLemma}

Since the lemma is obvious, we omit the corresponding proof.

Now, we begin the proof of Theorem \ref{theo_solution_X_special}. 

\begin{proof}
    First, we show that $\hat X \odot P = M$. 
    Note that we only need to prove that 
    \begin{equation}
        \hat{\bm x}^i P_i = \bm m^i .
    \end{equation}
    Note that 
    \begin{equation} \label{eq_property_F}
        (F_i)^{\bm p^i+} P_i = (F_i)^{\bm p^i+} .
    \end{equation}
    Combine the above equation and Eq. (\ref{solution_X_noiseless_appendix}), and we have
    \begin{equation}
    \begin{split}
        \bm x^i P_i & = \bm m^i (F_i)^{\bm p^i+} \hat D^{-1} P_i = \bm m^i (F_i)^{\bm p^i+} P_i \hat D^{-1} P_i \\
        & = \bm m^i (F_i)^{\bm p^i+} F_i = \bm m^i P_i = \bm m^i .
    \end{split}
    \end{equation}
    where we use the lemma and fact $\hat X \odot P = M$. 
    Hence, we prove that $\hat X \odot P = M$ holds.
    
    Now, we focus on how to prove $\|\nabla_X \mathcal{L}(\hat X, \hat V)\| \leq 2 \delta u \|M\|$. Let $G = \nabla_X \mathcal{L}(\hat X, \hat V)$ represent the gradient. Substitute Eq. (\ref{solution_X_noiseless_appendix}) into $G$, and we have
    \begin{equation}
        \begin{split}
            \bm g^i & = 2 \bm x^i D + \bm v^i P_i = 2 \bm m^i (F_i)^{\bm p^i+} \hat D^{-1}D - 2 \bm m^i (F_i)^{\bm p^i+} P \\
            & = 2 \bm m^i (F_i)^{\bm p^i+} \hat D^{-1} \hat D - 2 \delta \bm m^i (F_i)^{\bm p^i+} \hat D^{-1} - 2 \bm m^i (F_i)^{\bm p^i+} P \\
            & = 2 \bm m^i (F_i)^{\bm p^i+} - 2 \bm m^i (F_i)^{\bm p^i+} - 2 \delta \bm m^i (F_i)^{\bm p^i+} \hat D^{-1} \\
            & = -2 \delta \bm m^i (F_i)^{\bm p^i+} \hat D^{-1} \\
            & = -2 \delta \bm m^i (P_i \hat D^{-1} P_i)^{\bm p^i+}  \hat D^{-1} .
        \end{split}
    \end{equation}
    According to Lemma \ref{lemma_norm_inequality}, we have
    \begin{equation}
        \begin{split}
            \|\bm g^i\| & = 2 \delta \|\bm m^i (F_i)^{\bm p^i+}  \hat D^{-1}\| \\
            & \leq 2 \delta \|\bm m^i\| \| (F_i)^{\bm p^i+} \hat D^{-1}\| . \\
        \end{split}
    \end{equation}
    Now, we need to prove that there exists a constant $u$ such that $\| (F_i)^{\bm p^i+} \hat D^{-1}\| \leq u$. Let $S = \hat D^{-1}$, and we have
    \begin{equation}
        \begin{split}
            & \|(F_i)^{\bm p+} \hat D^{-1}\| \\
            = & \|(S)^{\bm p_+} S\| \\
            = & \|(S)^{\bm p_1+} P_1  P_2 \cdots S\| \\
            = & \|(S)^{\bm p_1+} P_1 S P_1 (S)^{\bm p_1+} P_2 S P_2 (S)^{\bm p_2+}  \cdots S\| \\
            \leq & \|(S)^{\bm p_1+} P_1 S P_1\| \cdot \|(S)^{\bm p_1+} P_2 S P_2\| \cdots \|(S)^{\bm p_t+} S\| .
        \end{split}
    \end{equation}
    where $[\bm p_1, \bm p_2, \cdots, \bm p_t]$ is a sequence, $\bm p_1 = \bm p$, and $\|\bm p_t\|_0 = n - 1$. $\bm p_{i+1}$ is constructed by replace a zero entry of $\bm p_i$. According to Lemma \ref{lemma_upper}, we have
    \begin{equation}
        \|(S)^{\bm p_t+} S\|^2 < (n - 1) + \frac{\|\bm {\hat d_k} \|_2^2}{\hat d_{kk}^2} .
    \end{equation}
    where $(\bm p_t)_k = 0$. 
    According to the Lemma \ref{lemma_assumption}, there exists a constant $c_t$ such that 
    \begin{equation}
        \frac{\|\bm {\hat d_k} \|_2^2}{\hat d_{kk}^2} \leq c_t .
    \end{equation}
    Hence, 
    \begin{equation}
        \|(S)^{\bm p_t+} S\|^2 < (n - 1) + n c_t^2 .
    \end{equation}
    Similarly, it is not hard to find that
    \begin{equation}
        \|(S)^{\bm p_i+} P_{i+1} S P_{i+1}\|^2 < (n - 1) + n c_i^2 .
    \end{equation}
    Let 
    \begin{equation}
        u = \sqrt{\prod \limits_{i} [(n-1) + n c_i^2]} .
    \end{equation}
    Hence, the theorem is proved. 
\end{proof}

\subsubsection{Proof of the General Case}
For any matrix $A$, let $A_{-i}$ denote a sub-matrix which is obtained by 
deleting the $i$-th column. Besides, define $\bm \pi_i$ as a vector where its 
$i$-th entry is 0 and others are 1.

The following lemma explains the situation when $D_{ii} = 0$.
\begin{myLemma} \label{lemma_Sigma0}
    As $D$ is computed by Algorithm \ref{alg_surrogate},
    $D_{ii} = 0$ if and only if $\bm x_i = 0$.
\end{myLemma}
\begin{proof}
    On the one hand, if $\bm x_i = 0$, then $(X^T X)_{ii} = 0$.
    Let $X^T X = V^T \Lambda V$. Note that 
    \begin{equation}
        (X^T X)_{ii} = \bm v_i^T \Lambda \bm v_i = 0.
    \end{equation}
    Clearly, $(X^T X)^{\frac{1}{2}} = V^T \Lambda^{\frac{1}{2}} V$. 
    Hence, we have 
    \begin{equation}
        ((X^T X)^{\frac{1}{2}})_{ii} = \bm v_i^T \Lambda^{\frac{1}{2}} \bm v_i = 0 .
    \end{equation}
    Similarly, 
    \begin{equation}
        ((X^T X)^{\frac{\dag}{2}})_{ii} = \bm v_i^T \Lambda^{\frac{\dag}{2}} \bm v_i .
    \end{equation}
    On the other hand, the proof of $D_{ii} = 0 \Rightarrow \bm x_i = 0$ is 
    similar.
\end{proof}

\begin{myCorollary}
    As $D$ is computed by Algorithm \ref{alg_surrogate}, if $D_{ii} = 0$, then 
    $\|\bm \gamma_i\|_0 = \|\bm \gamma^i\|_0 = 0$.
\end{myCorollary}

\begin{myLemma} \label{lemma_submatrix_sqrt}
    For any $i$, define $\bm \pi_i \in \mathbb{R}^{n + 1}$.
    Given an arbitrary matrix $Q \in \mathbb{S}_{+}^n$, let $A$ be an 
    $(n+1) \times (n+1)$ matrix where $[A]_{\bm \pi_i, \bm \pi_i} = Q$ and 
    $[A]_{\bm{\bar \pi}_i, \bm{\bar \pi}_i} = 0$. Then we will have $A \geq 0$,
    $[A^{\frac{1}{2}}]_{\bm \pi_i, \bm \pi_i} = Q^{\frac{1}{2}}$ and 
    $[A^{\frac{1}{2}}]_{\bm{\bar \pi}_i, \bm{\bar \pi}_i} = 0$.
\end{myLemma}

The proof is similar with Lemma \ref{lemma_submatrix_inverse}.

\begin{myLemma} \label{lemma_null_column}
    Let $X_* = \arg \min _{X \odot P = M} \|X\|_*$.
    If $\forall i, (i, j) \notin \Omega$, then $\|(\bm x_j)_*\|_0 = 0$.
\end{myLemma}

\begin{proof}
    Let 
    \begin{equation}
        R = \arg \min _{X_{-i} \odot P_{-i} = M_{-i}} \|X\|_* .
    \end{equation}
    If $(X_{-i})_* \neq R$, then we have $\|R\|_* <\|(X_{-i})_*\|_*$.
    Thereby we can construct a matrix $X_0$ that satisfies
    \begin{equation}
        [X_0]_{-i} = R, [X_0]_{i, \cdot} = 0.
    \end{equation}
    According to Lemma \ref{lemma_submatrix_sqrt}, we have 
    \begin{equation}
        \|X_0\|_* = {\rm tr}[(X_0^T X_0)^{\frac{1}{2}}] = {\rm tr}[(R^T R)^{\frac{1}{2}}] = \|R\|_* > \|(X_{-i})_*\|_* ,
    \end{equation}
    which results in a conflict. 

    Hence, $(X_{-i})_* = R$. In other words, $(\bm x_i)_* = 0$.
\end{proof}

\begin{proof} [Proof of Theorem \ref{theo_solution_X}]
    As is shown in Algorithm \ref{alg_surrogate} and Lemma \ref{lemma_Sigma0}, if $\forall i, (i, j) \notin \Omega$, then 
    $\bm x_i$, computed by
    \begin{equation}
        \bm x^i = \hat{\bm m}^i (\hat F_i)^{\bm p_i +} (H_i \hat D H_i)^{\bm p_i+},
    \end{equation}
    will always be 0. 
    According to Lemma \ref{lemma_null_column}, the neglect of the $i$-th 
    column is sound due to there being no observed entry. 

    To complete the proof of Theorem \ref{theo_solution_X}, we can convert the situation, where 
    $\exists i, D_{ii} = 0$, into the simple case which has been discussed 
    in the last subsection. Specifically speaking, provided that 
    $\forall k, (k, i) \notin \Omega$, the original problem can be transformed 
    into 
    \begin{equation}
        \begin{split}
        & \min \limits_{X_{-i}, [D]_{\bm \pi_i, \bm \pi_i}}  {\rm tr}(X_{-i} [D]_{\bm \pi_i, \bm \pi_i} X_{-i}^T), \\
         & s.t. ~ X_{-i} \odot P_{-i} = M_{-i}, [D]_{\bm \pi_i, \bm \pi_i} \geq 0,  {\rm tr}([D]_{\bm \pi_i, \bm \pi_i}^\dag) = 1.
        \end{split}
    \end{equation}
    The above transformation can be performed multiple times until all diagonal
    entries of $D$ are non-zero. Let $X_-$, $D_-$, and $\Gamma_-$ 
    (Lagrangian multipliers) be the corresponding matrices when all unobserved columns are removed.
    Clearly, we have 
    \begin{equation}
        \|\nabla_X \mathcal{L}(X, \Gamma)\| = \|\nabla_{X_-} \mathcal{L}(X_-, \Gamma_-) \|.
    \end{equation}
    Combining with the conclusion of Theorem \ref{theo_solution_X_special}, the theorem is thus proved.
\end{proof}

\subsection{Proof of Corollary \ref{theo_convergence}}
The proof of Corollary \ref{theo_convergence} relies on Theorem 
\ref{theo_equivalence_nuclear} while the latter does not rely on 
the former. Therefore, we just employ the conclusion of Theorem 
\ref{theo_equivalence_nuclear}.
\begin{proof}
    The proof is similar to the one for Theorem \ref{theo_equivalence_nuclear}.
    If each step decreases the objective value, the algorithm has to approach 
    a local minimum since the loss is lower-bounded. 
    
    Clearly, if $\delta \rightarrow 0$, then the solution that Algorithm 
    \ref{alg_surrogate} tends toward will be a valid 
    solution for problem (\ref{problem_convex}). 
    For this convex problem, 
    the solution is indeed its optimum. As is shown in the proof 
    of Theorem \ref{theo_equivalence_nuclear}, it is also the global optimum for the proposed surrogate.

    Hence, the theorem is proved. 
\end{proof}

\subsection{Proof of Theorem \ref{theo_equivalence_nuclear}}
 
\begin{proof}
    According to Proposition \ref{proposition}, problem (\ref{obj_no_graph}) in the main paper is a non-convex optimization problem. 
    From Theorem \ref{theo_equivalence}, we find that if the model converges, the optimal solutions, $X_*$ and $D_*$, should satisfy
    \begin{equation}
        \mathcal{R}((X_*)^T) \subseteq \mathcal{R}((D_*)^\dag) ,
    \end{equation}
    where $\mathcal{R}(\cdot)$ represents the space spanned by columns. 
    Let $G = D^\dag$.
    Interestingly, the above condition indicates that problem (\ref{obj_no_graph}) has the same optimum with a \textit{sub-problem},
    \begin{equation} \label{problem_convex}
        \begin{split}
            \min \limits_{X, G} & ~ {\rm tr}(X G^\dag X^T),\\
            s.t. & ~ X \odot P = M, {\rm tr}(G) = 1, G \in \mathbb{S}_+^{n}, \mathcal{R}(X^T) \subseteq \mathcal{R} (G). 
        \end{split}
    \end{equation}
    It should be emphasized that the feasible domain of the above problem is smaller than problem (\ref{obj_no_graph}) such that we call it a \textit{sub-problem}. 
    
    According to Page 76 and 651 of \cite{convex-optimization}, 
    the sub-problem is convex. 
    As we analyse above, $X_*$ and $D_*$ are also valid solution for the sub-problem 
    since $\mathcal{R}((X_*)^T) \subseteq \mathcal{R}((D_*)^\dag)$. 
    
    Via reductio, there is no solution which leads to a smaller value of the sub-problem. 
    Combining with the convexity of the sub-problem, $(X_*, D_*)$ is the optimum of it. 
    Suppose that there exists $(X_*, D_0)$ that satisfies the constraints of problem (\ref{obj_no_graph}) 
    but does not obey $\mathcal{R}((X_*)^T) \subseteq \mathcal{R}(D_0^\dag)$. 
    Then, we can set 
    \begin{equation}
        \hat D_0 = (\frac{((X_*)^T X_*)^\frac{1}{2}}{{\rm tr}(((X_*)^T X_*)^{\frac{1}{2}})})^{\dag} ,
    \end{equation}
    such that ${\rm tr}(X_* \hat D_0 (X_*)^T) \leq {\rm tr}(X_* D_0 (X_*)^T)$. 
    Since ${\rm tr}(X_* D_0 (X_*)^T) \leq {\rm tr}(X_* \hat D_0 (X_*)^T)$, 
    ${\rm tr}(X_* D_0 (X_*)^T) = {\rm tr}(X_* \hat D_0 (X_*)^T)$. 
    In other words, $(X_*, \hat D_0)$ is also an optimal solution of problem (\ref{obj_no_graph}). 
    
    Furthermore, the optimal solution of the sub-problem should be the optimal solution of problem (\ref{obj_no_graph}). 
    
    Now, we need to show the following equation,
    \begin{equation}
        \|X_*\|_*^2 = \min \limits_{X \odot P = M} \|X\|_*^2  .
    \end{equation}
    Suppose that $\tilde X = \arg \min \limits_{X \odot P = M} \|X\|_*^2$.
    Since $X_* \odot P = M$, 
    $\|X_*\|_*^2 \geq  \|\tilde X\|_*^2$.
    Similarly, if we set 
    \begin{equation}
        \tilde D = (\frac{(\tilde X^T \tilde X)^\frac{1}{2}}{{\rm tr}((\tilde X^T \tilde X)^{\frac{1}{2}})})^{\dag} ,
    \end{equation}
    then $(\tilde X, \tilde D)$ is a solution of the sub-problem. 
    According to Theorem \ref{theo_solution_D}, we have $\|X_*\|_*^2 \leq \|\tilde X\|_*^2$. 
    Therefore, we have proven $\|\tilde X\|_*^2 = \|X_*\|_*^2$, 
    which means $X_* = \tilde X$.
    
    Hence, the theorem is proved. 
\end{proof}

\subsection{Proof of Theorem \ref{theo_graph}}
\begin{proof}

    Let $l_{ij} = \|\bm x_i - \bm x_j\|_2^2$, and the problem to solve is 
    \begin{equation}
        \begin{split}
            & \min \limits_{S \textbf{1} = \textbf{1}, S \geq 0} {\rm tr}(X L X^T) + \|\bm \mu^T S\|_F^2 \\
            \Leftrightarrow & \min \limits_{S \textbf{1} = \textbf{1}, S \geq 0} \sum \limits_{i}^n \sum \limits_{j}^n l_{ij} S_{ij} + \mu_i^2 \sum \limits_{j}^n S_{ij}^2 .
        \end{split}
    \end{equation}

And the problem is equivalent to solve the following $n$ subproblems individually, 
\begin{equation}
    \label{obj_alpha}
    \min \limits_{\bm s^i \textbf{1} = 1, \bm s^i \geq 0} \sum \limits_{j}^n l_{ij} S_{ij} + \mu_i^2 \sum \limits_{j} S_{ij}^2 .
\end{equation}
More abstractly, every subproblem is equivalent to
\begin{equation}
    \min \limits_{\bm \alpha^T \textbf{1}_n = 1, \bm \alpha \geq 0} \|\bm \alpha + \frac{\bm f}{2 \lambda}\|_2^2 .
\end{equation}
Similarly, the Lagrangian of the above equation is
\begin{equation}
    \mathcal{L}_\alpha  = \|\bm \alpha + \frac{\bm f}{2 \lambda}\|_2^2 + \xi (1 - \sum \limits_{i=1}^n \alpha_i) - \sum \limits_{i=1}^n \beta_i \alpha_i ,
\end{equation}
where $\xi$ and $\beta_i$ is Lagrangian variables. The KKT conditions are given as
\begin{equation}
    \left\{
    \begin{array}{l}
        \frac{\partial \mathcal{L}_\alpha}{\partial \alpha_i} = \alpha_i + \frac{f_i}{2 \lambda} - \xi - \beta_i = 0 \\
        \beta_i \alpha_i = 0 \\
        \sum \limits_{i=1}^n \alpha_i = 1, \beta_i \geq 0, \alpha_i \geq 0 .
    \end{array}
    \right.
\end{equation}
Then we consider the following cases
\begin{equation}
    \left\{
    \begin{array}{l}
        \alpha_i = 0 \Rightarrow \xi - \frac{f_i}{2 \lambda} = - \beta_i \leq 0 \\
        \alpha_i \geq 0 \Rightarrow \alpha_i = \xi - \frac{f_i}{2 \lambda} ,
    \end{array}
    \right.
\end{equation}
which means
\begin{equation}
    \alpha_i = (\xi - \frac{f_i}{2 \lambda})_+ .
\end{equation}
Without loss of generality, assume that $f_1 \leq f_2 \leq \cdots \leq f_n$. If $\xi - \frac{f_{k+1}}{2 \lambda} \leq 0 < \xi - \frac{f_{k}}{2 \lambda}$, then $\alpha$ is $k$-sparse. Due to $\sum \limits_{i=1}^n \alpha_i = 1$, we have
\begin{equation}
    \sum \limits_{i=1}^n \alpha_i = k \xi - \sum \limits_{i=1}^k \frac{f_i}{2 \lambda} = 1 ,
\end{equation}
which means
\begin{equation}
    \xi = \frac{1}{2 k \lambda} \sum \limits_{i=1}^k f_i + \frac{1}{k} .
\end{equation}
Combine with our assumption and we have
\begin{equation}
    \begin{array}{l}
     \frac{1}{2 k \lambda} \sum \limits_{i=1}^k f_i + \frac{1}{k} - \frac{f_{k+1}}{2 \lambda} \leq  0 <  \frac{1}{2 k \lambda} \sum \limits_{i=1}^k f_i + \frac{1}{k} - \frac{f_{k}}{2 \lambda} \\
    \Rightarrow  \frac{f_k}{2 \lambda} <  \frac{1}{2 k \lambda} \sum \limits_{i=1}^k f_i + \frac{1}{k}  \leq \frac{f_{k+1}}{2 \lambda} \\
    \Rightarrow \frac{k f_k}{2} - \frac{1}{2} \sum \limits_{i=1}^k f_i < \lambda \leq \frac{k f_{k+1}}{2} - \frac{1}{2} \sum \limits_{i=1}^k f_i .
    \end{array}
\end{equation}
Hence, if $\lambda$ is set within the above range, then $\bm \alpha$ will be $k$-sparse.  In other words, $\lambda$ is converted into the amount of neighbors $k$. In classification tasks, as the amount of $k$ frequently takes a small proportion, $k$ is set as a large value and named as the number of activation samples. 

If we simply set $\lambda$ as its upper bound, \textit{i.e.,} $\lambda = \frac{k f_{k+1}}{2} - \frac{1}{2} \sum \limits_{i=1}^k f_i$, we have
\begin{equation}
    \label{update_alpha}
    \begin{split}
        \alpha_i & = (\frac{\sum \limits_{i=1}^k f_i + 2\lambda}{2 k \lambda} - \frac{f_i}{2 \lambda})_+  = (\frac{f_{k+1} - f_i}{k f_{k+1} - \sum \limits_{i=1}^k f_i})_+ .
    \end{split}
\end{equation}

Hence, the theorem is proved. 

\end{proof}

\bibliographystyle{IEEEtran.bst}

\bibliography{mc_can}

\begin{IEEEbiographynophoto}{Xuelong Li} (M'02-SM'07-F'12) 
    is a Full Professor with the School of Artificial Intelligence, OPtics and ElectroNics (iOPEN), Northwestern Polytechnical University, Xi'an, China. 
\end{IEEEbiographynophoto}

\begin{IEEEbiography}[{\includegraphics[width=1in,height=1.25in,clip,keepaspectratio]{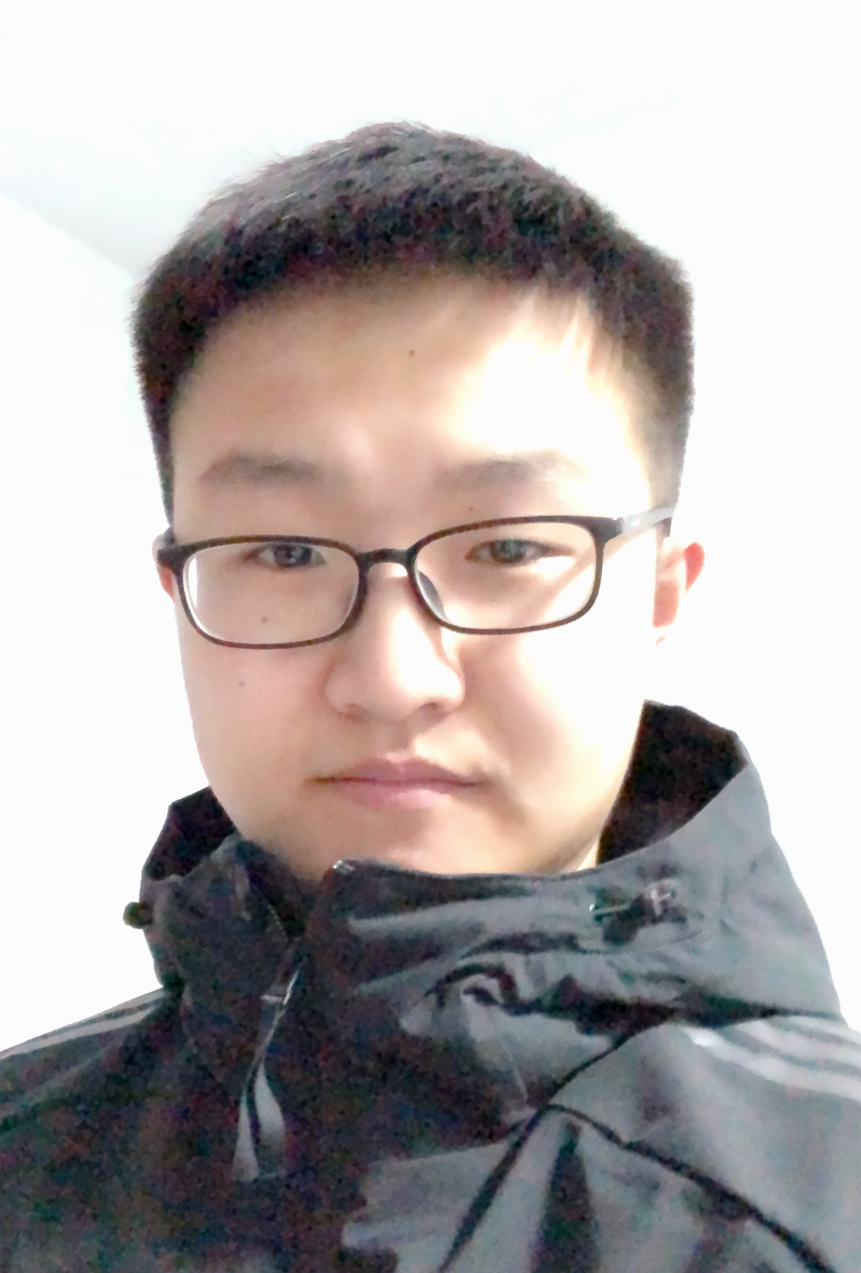}}]{Hongyuan Zhang}
    received the B.E. degree in software engineering from Xidian University, Xi'an, China in 2019. 
    He is currently pursuing the Ph.D. degree from the School of Computer Science and the School of Artificial Intelligence, OPtics and ElectroNics (iOPEN), Northwestern Polytechnical University, Xi'an, China. 
\end{IEEEbiography}
    
\begin{IEEEbiography}[{\includegraphics[width=1in,height=1.25in,clip,keepaspectratio]{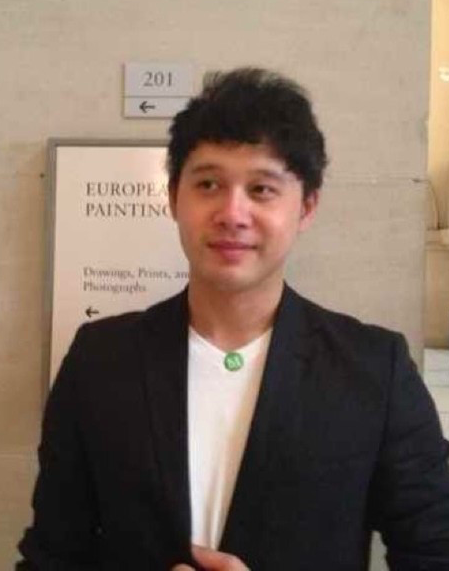}}]{Rui Zhang} (M'19)
    received the Ph.D degree in computer science at Northwestern Polytechnical University, Xi'an, China in 2018. 
    He currently serves as an Associate Professor with the School of Artificial Intelligence, OPtics and ElectroNics (iOPEN), Northwestern Polytechnical University, Xi'an, China.
\end{IEEEbiography}

\end{appendices}

\end{document}